\documentclass[runningheads]{llncs}

\usepackage[T1]{fontenc}
\usepackage{graphicx}
\usepackage[hyphens]{url}
\usepackage{hyperref}
\usepackage{xcolor}

\usepackage{xr-hyper}

\usepackage{amsmath}
\usepackage{amssymb}
\usepackage{amsfonts}
\usepackage{subcaption}

\usepackage{algorithm}
\usepackage{algorithmic}

\usepackage{wrapfig}

\usepackage{thmtools}
\usepackage{thm-restate}

\usepackage{newfloat}
\usepackage{listings}
\DeclareCaptionStyle{ruled}{labelfont=normalfont,labelsep=colon,strut=off}
\floatstyle{ruled}
\newfloat{listing}{tb}{lst}{}
\floatname{listing}{Listing}

\def\CX{\mathcal{X}} 
\def\CA{\mathcal{A}}

\def\CZ{\mathcal{Z}}
\def\Zspace{\mathcal{Z}}

\newcommand{\emathbb}[1]{\ensuremath{\underset{{#1}}{\mathbb E}}}
\def\topo{\tau}

\def\policy{\pi}

\def\policyfstar{\policy^{*}}

\def\hbar{\bar{h}}

\newcommand{\probz}[1]{\ensuremath{\mathbb{P}_{Z}({#1})}}

\DeclareMathOperator*{\argmax}{arg\,max}

\title{Open-loop POMDP Simplification and Safe Skipping of Replanning with Formal Performance Guarantees}
\titlerunning{Open-loop POMDP Simplification with Formal Guarantees}

\author{Da Kong \and Vadim Indelman}
\authorrunning{D. Kong and V. Indelman}
\institute{Technion - Israel Institute of Technology \\ \email{da-kong@campus.technion.ac.il, vadim.indelman@technion.ac.il}}

\date{}

\begin{document}

\maketitle

\begin{abstract}
Partially Observable Markov Decision Processes (POMDPs) provide a principled mathematical framework for decision-making under uncertainty. However, the exact solution to POMDPs is computationally intractable. In this paper, we address the computational intractability by introducing a novel framework for adaptive open-loop simplification with formal performance guarantees. Our method adaptively interleaves open-loop and closed-loop planning via a topology-based belief tree, enabling a significant reduction in planning complexity. The key contribution lies in the derivation of efficiently computable bounds which provide formal guarantees and can be used to ensure that our simplification can identify the immediate optimal action of the original POMDP problem. Our framework therefore provides computationally tractable performance guarantees for macro-actions within POMDPs. Furthermore, we propose a novel framework for safely skipping replanning during execution, supported by theoretical guarantees on multi-step open-loop action sequences. To the best of our knowledge, this framework is the first to address skipping replanning with formal performance guarantees. Practical online solvers for our proposed simplification are developed, including a sampling-based solver and an anytime solver.
Empirical results demonstrate substantial computational speedups while maintaining provable performance guarantees, advancing the tractability and efficiency of POMDP planning.

\end{abstract}

  \section{Introduction}
Partially Observable Markov Decision Processes (POMDPs) constitute a fundamental mathematical framework for sequential decision-making under uncertainty. Despite the theoretical elegance and broad applicability across diverse domains,
POMDPs suffer from significant computational intractability, known as the \textit{curse of dimensionality} and the \textit{curse of history}.

Extensive research has focused on approximation methods to improve POMDP tractability~\cite{Lee07nips,Kurniawati08rss,Porta06jmlr,Lim23jair,Elimelech22ijrr,LevYehudi24aaai,Barenboim26aij}. Open-loop planning has emerged as a promising approach, reducing the belief tree complexity by eliminating observation branches~\cite{Flaspohler20nips,Amato19jair,He11jair}. 
Unlike closed-loop methods that adapt to observations, open-loop approaches execute predefined action sequences without gathering observations during execution, also known as macro-actions~\cite{Flaspohler20nips}.

However, open-loop planning and macro-actions typically yield suboptimal solutions without any performance guarantees. Although prior work~\cite{Flaspohler20nips} offers bounds for macro-actions, these rely on the computationally intractable Value of Information (VoI), restricting their practical online use.
This limitation exposes a fundamental gap in the existing literature: the absence of tractable performance guarantees for open-loop approximations relative to the original POMDP---a gap that this work aims to address.

Furthermore, conventional online POMDP planners replan after executing only the first action, even in the open-loop approximation setting~\cite{He11jair}. Skipping replanning can be an effective strategy to simplify POMDP planning at the execution level. While recent learning-based approaches have begun to address this challenge~\cite{Honda24icra}, the establishment of formal performance guarantees for safely skipping replanning remains an open problem that we tackle in this work.



Specifically, in this work, we propose a novel framework for adaptive open-loop POMDP simplification at two levels: first, it enables POMDP planning simplification while maintaining performance guarantees that ensure the identification of the same optimal action at the root as the original POMDP; second, it enables safely skipping replanning with formal performance guarantees.


 \begin{wrapfigure}{r}{0.33\textwidth}
    \vspace{-0.6\baselineskip}
    \centering
    \includegraphics[width=0.31\textwidth]{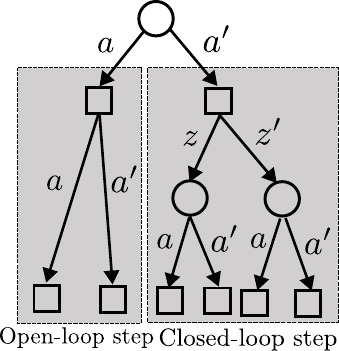}
    \caption{A hybrid belief tree demonstrating the computational advantage of open-loop planning. The left branch employs open-loop action sequences, while the right branch utilizes traditional closed-loop planning.}
    \label{fig:open-loop-hybrid}
    \vspace{-0.8\baselineskip}
\end{wrapfigure}

More specifically, at the planning level, we propose a novel method that adaptively introduces the open-loop mode to some belief nodes, excluding observations to simplify POMDP planning with performance guarantees. As illustrated in Fig.~\ref{fig:open-loop-hybrid}, the incorporation of open-loop planning can dramatically reduce the size of the belief tree, which directly corresponds to a reduction in planning complexity.
Our method adaptively chains open-loop and closed-loop steps.
We derive novel computationally tractable bounds that relate this simplified POMDP to the original POMDP. Importantly, our bounds only depend on the simplified problem, enabling POMDP simplification with computationally tractable performance guarantees. 
This, in turn, enables online adaptation between open-loop and closed-loop branches in the belief tree by utilizing the corresponding bounds. To the best of our knowledge, this constitutes the first work to provide computationally tractable formal guarantees for incorporating open-loop planning, and equivalently, macro-actions into POMDPs.


At the execution level, we propose the first framework for skipping replanning in POMDPs with formal performance guarantees.  
Specifically, we establish novel bounds on the posterior $Q$-values at future planning sessions based solely on the current planning session. We show that whenever these bounds satisfy certain conditions, the agent can safely skip replanning and execute multi-step open-loop actions that are guaranteed to be identical to the immediate optimal actions of the original POMDP problem, obtained through explicit replanning.

Further, we employ a sampling-based solver and an anytime solver for the derived bounds and demonstrate our method's effectiveness in simulated POMDP environments. Our approach achieves significant computational speedups while maintaining provable performance guarantees, highlighting its potential for real-world applications.

To summarize, our main contributions are as follows:
\begin{itemize}
    \item We introduce the first POMDP simplification framework that adaptively incorporates open-loop steps and provides formal performance guarantees for both planning and execution. 
    \item For planning, we develop novel efficiently computable bounds for introducing adaptive open-loop steps in POMDP planning that yield provable performance guarantees for our simplification method.
    \item For execution, we derive novel bounds for executing multi-step open-loop steps in POMDPs, providing guarantees for safely skipping replanning. 
    \item We develop practical sampling-based and anytime MCTS-style solvers for our simplification framework, namely, AT-SparsePFT and AT-POMCP. 
    They demonstrate substantial computational improvements through empirical evaluation while maintaining theoretical guarantees.
\end{itemize}

\section{Related Work}

For general background on POMDPs in robotics, we refer to surveys~\cite{Kurniawati22ar,Lauri22tro}. This section will focus on two specific aspects most relevant to this work: open-loop planning of POMDP and skipping replanning.

\paragraph{Open-Loop Planning.}
Practical POMDP solvers employ various approximation techniques to address the computational intractability, including approximate belief representations~\cite{Roy05jair,Lee07nips,Lim23jair} and memory-based approximations~\cite{Kara22jmlr,Subramanian19cdc,Patil24aistat}. Beyond approximation, simplification of POMDP with formal performance guarantees is essential for safety-critical tasks, including research on simplifying the observation model~\cite{LevYehudi24aaai} and simplifying the state-observation space~\cite{Barenboim26aij}. 

Within the POMDP literature, open-loop planning is often formulated as \textit{macro-actions}.
He et al.~\cite{He11jair} established the first formal analysis of macro-action planning in POMDPs, providing crucial theoretical foundations for this approach. Subsequent work by Amato et al.~\cite{Amato19jair} extended this framework to decentralized POMDPs, demonstrating the broader applicability of macro-action abstractions in multi-agent systems. 

Recent theoretical advances include Flaspohler et al.'s~\cite{Flaspohler20nips} novel approach leveraging the Value of Information (VoI) to synthesize open-loop action sequences. While theoretically elegant, the substantial computational complexity of VoI calculations presents practical limitations for POMDP simplification. From a robotic motion control perspective, Majumdar et al.~\cite{Majumdar23ijrr} derived performance bounds for robot motion models that implicitly provide guarantees for open-loop planning in POMDPs, though these results do not directly address the POMDP simplification problem.

Contemporary research has shifted toward data-driven methods for macro-action sequence generation~\cite{Lee21rss,Liang23iros}, demonstrating promising empirical results.
However, these learning-based approaches currently lack the theoretical performance guarantees necessary for safety-critical applications.

\paragraph{Skipping Replanning.}
Conventional online planning pipelines employ a cyclical scheme of planning, execution, and replanning. However, the strategic decision of when to forgo replanning has received limited attention. Traditional approaches rely on hand-crafted replanning strategies, which often prove inflexible and suboptimal. Honda et al.~\cite{Honda24icra} recently addressed this gap through deep reinforcement learning, proposing adaptive replanning strategies for dynamic environments. However, their method lacks performance guarantees. To the best of our knowledge, we present the first work with  theoretical analysis of replanning strategies in POMDPs with provable performance bounds.

\section{POMDP Preliminaries}
The basic model of POMDP is defined as a tuple:
$\langle\CX, \CA, \Zspace, \mathbb{P}_{T}, \mathbb{P}_{Z}, r, b_0\rangle$, where $\CX$ is the state space, $\CA$ is the action space, $\Zspace$ is the observation space.
 The transition model (or motion model) is defined as $\mathbb{P}_{T}(x_{k+1}|x_k, a_k)$, which describes the probabilistic transition of the state from $x_k\in\CX$ to $x_{k+1}\in\CX$ under a certain action $a_k \in \CA$. The observation model is defined as $\probz{z_k|x_k}$, which describes the probability of observation $z_k\in\Zspace$ given a certain state $x_k\in\CX$.
 The reward function $r$ is considered to be state-dependent, $r(b,a)=\mathbb{E}_{x|b}[r(x,a)]$, and is  bounded, i.e., $r \in [-R_{\max}, R_{\max}]$. $b_0$ is the initial belief.

Given that the true state is uncertain, a belief is maintained to represent the distribution of the current state given the history.
The belief at any time instant $k$ is defined as $b_k \triangleq P(x_{k}| h_{k}) $, where $h_{k}$ is the history at that time, defined as $h_{k} \triangleq \{z_{1:k},a_{0:k-1}\}$. A propagated history without the latest observation $h_k^{-}$ is defined as $h_k^{-} = \{z_{1:k-1},a_{0:k-1}\}$, and the corresponding propagated belief is $b^-_k \triangleq P(x_k|h^-_k)$.

A deterministic policy function is defined as $\pi:\mathcal{H}\mapsto\CA$, which decides actions based on the history of beliefs. While a stochastic policy maps the history-action pairs to a probability as: $\pi: \mathcal{H}\times\CA\mapsto [0,1]$. The value function for a certain policy $\pi$ over the planning horizon $L$ is defined as the summation of all expected rewards:
$    
V^\pi(b_k)= r(b_k, \pi_k(b_k)) + \sum_{i=k+1}^{k+L}\emathbb{z_{k+1:i}}\big[ r(b_i,\pi_i(b_i))\big].
$
The goal of a POMDP is to find the optimal policy $\pi^*$ that maximizes the value function. The optimal policy can be calculated by the Bellman optimality as:
$
V^{\pi*}(b_k) = \max_{a_k} \Big[ r(b_k, a_k) + \emathbb{z_{k+1}} V^{\pi*}(b_{k+1})\Big].
$

	\section{Open-loop as Simplification of POMDP}
	\label{sec:open-loop-method}
In this section, we present a novel framework for POMDP simplification that adaptively interleaves open-loop and closed-loop planning with formal performance guarantees. We derive computationally tractable bounds wherein both upper and lower bounds can be efficiently computed, depending solely on our simplification method and serving as the foundation for our performance guarantees. We denote the upper and lower bounds as $ub$ and $lb$, 
\begin{align}
	lb({\tau}, b_0, a_0) \leq Q^{\pi*}(b_0,a_0) \leq  ub({\tau}, b_0, a_0),
	\label{eq:ub-lb-simplification}
\end{align}
where $\tau$ is a simplified topology, which can be seen as a control parameter of simplification of the POMDP problem and will be defined later in this section.


\begin{figure}[t]
    \centering
    \begin{subfigure}[b]{0.25\linewidth}
        \centering
        \includegraphics[width=\linewidth]{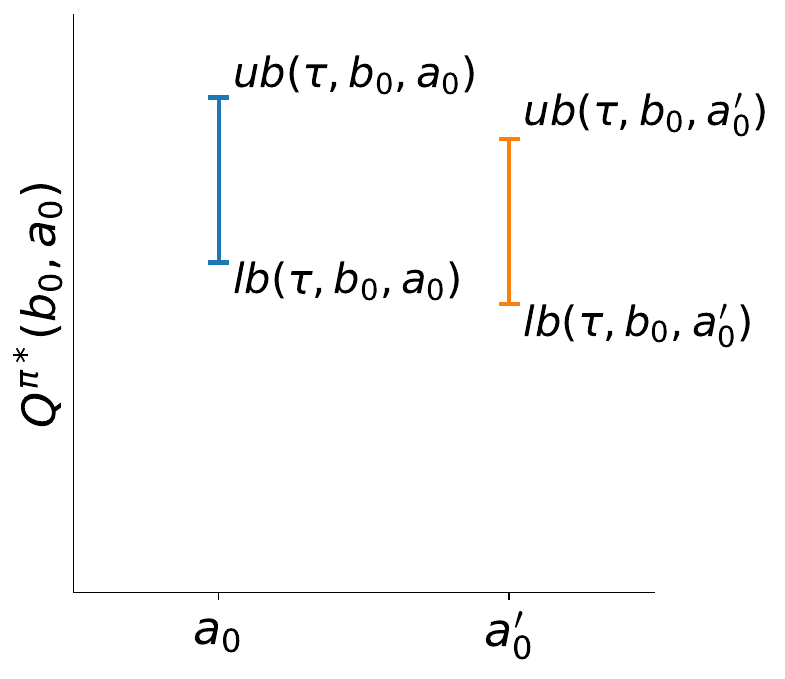}
        \caption{}
        \label{fig:two-guarantees-left}
    \end{subfigure}\quad \qquad
    \begin{subfigure}[b]{0.25\linewidth}
        \centering
        \includegraphics[width=\linewidth]{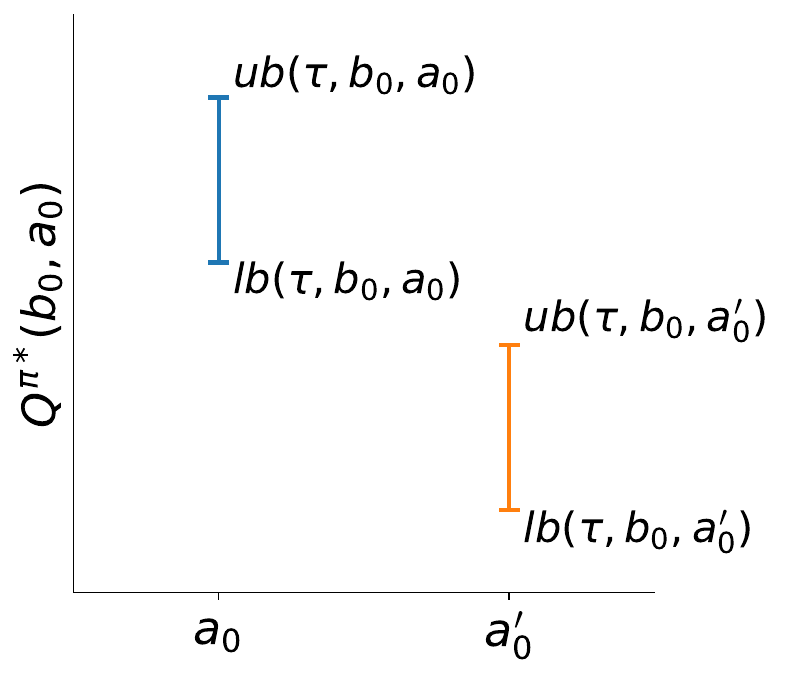}
        \caption{}
        \label{fig:two-guarantees-right}
    \end{subfigure}
    \caption{Illustration of two cases for performance guarantees: (a) overlapping bounds necessitating topology refinement, and (b) non-overlapping bounds allowing for optimal action determination.}
    \label{fig:two-guarantees}
\end{figure}


Such bounds can be used to determine the optimal action $a_0$ of the original POMDP problem in case there is no overlap between the highest lower bound and the second-highest upper bound.  If there is overlap, we need to refine the topology $\tau$ to achieve non-overlapping bounds. Fig.~\ref{fig:two-guarantees} illustrates these two cases.


\subsection{Definition}
\label{subsec:definition}
In POMDP, the policy space for both deterministic and stochastic policies is dependent on the history space. For the common full POMDP,  the history will include the latest observation. In contrast, open-loop planning does not incorporate observations, leading to a simplified history space, which is the starting point of our simplification.

We define a history updater for \textit{single-step} open-loop (OL) planning as $\psi^{OL}(h_k,a_k,z_{k+1})=\{h_k,a_k\}$, for any possible observation $z_{k+1}\in\mathcal{Z}$, history $h_k\in\mathcal{H}_k$, and action $a_k\in\mathcal{A}$.
This is in contrast to the standard \textit{single-step} closed-loop (CL) history updater, which we denote as $\psi^{CL}(h_k,a_k,z_{k+1})=\{h_k,a_k,z_{k+1}\}$. 
By iterating the single-step open-loop history updater, we can get the history for fully open-loop planning as: $h^{FOL}_k=\{a_{0:k-1}\}$. On the other hand, original POMDPs always use the closed-loop history updater, with a history of $h^{Full}_k=\{z_{1:k},a_{0:k-1}\}$. 

In order to incorporate open-loop and closed-loop steps adaptively, we introduce a topology $\tau$ to define the history and belief tree structure. We denote its history as Adaptive Open-Loop (AOL) history, $h_k^{\text{AOL},\tau}$.
The topology $\tau$ is defined as a set of binary indicator functions $\beta^{\text{AOL},\tau}(h_k^{\text{AOL},\tau})$. Specifically, $\beta^{\text{AOL},\tau}(h_k^{\text{AOL},\tau})=1$ indicates that a single-step open-loop planning is adopted for the history $h^-_k$, while $\beta^{\text{AOL},\tau}(h_k^{\text{AOL},\tau})=0$ indicates a single-step closed-loop planning.
The \textit{augmented} history update function $\psi$ is defined as: \
\begin{align}
	\label{eq:history-updater}
	&h^{\text{AOL},\tau}_k =	\psi^{AOL} (\beta^{\text{AOL},\tau}, h^{\text{AOL},\tau}_{k-1}, a_{k-1}, z_k )  \\& =  \begin{cases}
		\psi^{OL}(h^{\text{AOL},\tau}_{k-1},a_{k-1},z_{k}), & \text{if $\beta^{\text{AOL},\tau}(h^{\text{AOL},\tau}_{k-1})=1$}, \\	
		\psi^{CL}(h^{\text{AOL},\tau}_{k-1},a_{k-1},z_{k}), & \text{if $\beta^{\text{AOL},\tau}(h^{\text{AOL},\tau}_{k-1})=0$}. \nonumber
	\end{cases}
\end{align}
Here, $h^{\text{AOL},\tau}_k$ denotes the augmented history at time step $k$ under topology $\tau$, which can be either open-loop or closed-loop depending on the indicator functions $\beta^{\text{AOL},\tau}$.
By recursively applying \eqref{eq:history-updater} from the initial belief to the planning horizon $L$, we construct a complete history space $\mathcal{H}^{\text{AOL}, \tau}_t=\{h_t^{\text{AOL},\tau}\}$ and the corresponding belief tree $\mathbb{T}^{\text{AOL},\tau}=\{\mathcal{H}^{\text{AOL},\tau}_t: 1\leq t\leq L\}$ for topology $\tau$. This formulation encompasses two extreme cases: fully open-loop planning, where $\beta^{\text{AOL},\tau} = 1$ for all nodes, and the original POMDP formulation with fully closed-loop planning, where $\beta^{\text{AOL},\tau} = 0$ for all nodes.


This belief tree topology definition is an extension of~\cite{Kong24isrr}, where the topology is used to switch between original and simplified observation spaces and models. In contrast, we use topology to chain open-loop and closed-loop steps, contributing to an adaptive simplification of history.

Further, the augmented  policy is defined as $\pi^{\text{AOL},\tau}_t: \mathcal{H}^{\text{AOL},\tau}_t \mapsto \mathcal{A}$. It maps augmented histories to actions, enabling adaptive switching between open-loop action sequences and closed-loop policy steps according to the topology $\tau$. The corresponding $Q$-function for adaptive policy is defined as $
Q^{\pi^{\text{AOL},\tau*}}(b_0,a_0)= \mathbb{E}_{x_0|b_0} [ r(x_0,a_0) + \mathbb{E}_{z_1|b_0,a_0}\max_{a_1} Q^{\pi^{\text{AOL},\tau*}}(b(h^{\text{AOL},\tau}_1),a_1) ]$. If $\beta^{\text{AOL},\tau}(h^{\text{AOL},\tau}_1) = 1$, as determined by the adaptive topology, the expectation over $z_1$ cancels out.

We now introduce another level of open-loop planning where the observations are not considered but full observability is assumed.
We introduce the Adaptive Fully-Observable (AFO) policy ${\pi}^{\text{AFO}, \tau}$, which employs topology $\tau$ to control observability—assuming full observability without dependence on observations at designated nodes while maintaining partial observability at others.
The corresponding adaptive history $h^{\text{AFO}, \tau}_t$ is updated according to:
\begin{align}
	    \label{eq:history-update-fo}
    &h^{\text{AFO}, \tau}_t = \psi^{\text{AFO}}(\beta^{\text{AFO},\tau}, h^{\text{AFO}, \tau}_{t-1}, a_{t-1}, z_t) \\
    &= \begin{cases}
        h^{\text{AFO}, \tau}_{t-1} a_{t-1} x_t, & \text{if } \beta^{\text{AFO},\tau}(h^{\text{AFO}, \tau}_{t-1}) = 1, \nonumber \\
		\psi^{CL}(h^{\text{AFO},\tau}_{t-1},a_{t-1},z_{t}), & \text{if } \beta^{\text{AFO},\tau}(h^{\text{AFO}, \tau}_{t-1}) = 0.
    \end{cases}
\end{align}
Here, the indicator function $\beta^{\text{AFO},\tau}(\cdot)$ within the topology $\tau$ determines the observability mode: $\beta^{\text{AFO},\tau} = 1$ denotes full observability at the belief node, while $\beta^{\text{AFO},\tau} = 0$ denotes partial observability. Similarly, the $Q$-function for the AFO policy is defined as \[Q^{{\pi}^{\text{AFO},\tau*}}(b_0,a_0)= \mathbb{E}_{x_0|b_0} [ r(x_0,a_0) +\mathbb{E}_{x_1|b_0,a_0}\mathbb{E}_{z_1|x_1} \max_{a_1} Q^{{\pi}^{\text{AFO},\tau*}}(b(h^{\text{AFO},\tau}_1),a_1) ],\] where if $\beta^{\text{AFO},\tau}(h^{\text{AFO},\tau}_1)= 1$, the expectation of $z_1$ will cancel out but the expectation of $x_1$ will remain since the true state is included in $h^{\text{AFO},\tau}_1$. 



\subsection{Performance Guarantees}
\label{subsec:performance-guarantees}
For the bounds in~\eqref{eq:ub-lb-simplification}, we propose to adopt the $Q$-function of the optimal adaptive open-loop policy $\pi^{\text{AOL},\tau*}$ for the lower bound $lb$, and the optimal adaptive fully-observable policy $\pi^{\text{AFO},\tau*}$ for the upper bound $ub$.
We now present our main result that establishes these upper and lower bounds over the optimal $Q$-function of the original POMDP problem.
\begin{restatable}{theorem}{TheoremOneRestate}
	\label{theorem:open-loop2}  
	Let $\policyfstar$ denote the optimal policy of the original POMDP. Consider some topology $\topo_U$ and $\topo_L$, and denote the topology-dependent optimal augmented open-loop policy as ${\pi}^{AOL,\tau_L*}$. In the same way, we denote the topology-dependent optimal adaptive fully-observable policy as $\pi^{\text{AFO}, \tau_U*}$. Then, 
	\begin{flalign}
	Q^{\pi^{\text{AOL},\tau_L*}}(b_0,a_0) \leq	Q^{\policyfstar}(b_0,a_0)  \leq Q^{\policy^{\text{AFO},\tau_U*}}(b_0,a_0).
	\label{eq:BoundsThm1_a}
	\end{flalign}
\end{restatable}

\begin{proof}
	We provide the proof in Appendix~\ref{proof:theorem_1}. 
\end{proof}
Since the topologies $\tau_U$ and $\tau_L$ are always specified together and share the same property, for simplicity, we denote $\tau=(\tau_U,\tau_L)$ and use $\tau$ as shorthand for this pair from now on when the meaning is unambiguous.
Precisely, $ub$ and $lb$ may be induced by different (but jointly specified) topologies: the upper bound uses $\tau_U$ through the AFO policy, while the lower bound uses $\tau_L$ through the AOL policy. Since we always consider them as a bundle, we use a slight abuse of notation and use $\tau$ as shorthand from now on since the distinction is clear from context.

In particular, based on Theorem \ref{theorem:open-loop2}, the bounds in \eqref{eq:ub-lb-simplification} are defined as
\begin{align}\label{eq:BoundsThm1}
	ub(\tau, b_0,a_0) \triangleq Q^{\pi^{\text{AFO},\tau*}}(b_0,a_0) , \ \ lb(\tau, b_0,a_0) \triangleq Q^{\pi^{\text{AOL},\tau*}}(b_0,a_0).
\end{align}
As long as the bounds are easier to compute than the original POMDP and have no overlap as shown in Fig.~\ref{fig:two-guarantees-right}, we can use them to determine the optimal action $a_0$ of the original POMDP problem and achieve the simplification of POMDP with formal performance guarantees.

When the bounds \eqref{eq:ub-lb-simplification} for different actions overlap  under topology $\tau$, we have to explore an alternative topology $\tau'$ to achieve non-overlapping bounds through an iterative process, which continues until we identify the optimal action (see Fig.~\ref{fig:two-guarantees}). The transition process typically switches some of the belief nodes from a simplified mode to the original POMDP mode. This iterative process monotonically tightens the bounds; moreover, the bounds converge to the optimal $Q$-function of the original POMDP after a finite number of iterations. We provide a detailed analysis of these properties in Appendix~\ref{appendix:bound-analysis}.


\subsection{Integration with Online POMDP Solvers}
\label{subsec:solvers}
To bridge the gap between theoretical analysis and practical implementation, we propose to use online POMDP solvers to estimate the bounds for our proposed simplification method and adapt topology online. In this work we consider  two such solvers: sparse-sampling-style solver, Sparse-PFT~\cite{Lim23jair}, and an anytime MCTS solver, which estimates the bounds and adapts the topology in an anytime manner. Specifically, in this section we consider a particle belief POMDP setting, and introduce two solvers, AT-SparsePFT and AT-POMCP. The results could be  extended to POMDPs with theoretical beliefs, i.e.~belief-MDP, following~\cite{Lim23jair}.
Utilizing these solvers enables practical applications of our proposed adaptive open and closed loop online planning framework with formal guarantees.

\subsubsection{AT-SparsePFT.} 
The SparsePFT algorithm represents beliefs using weighted particles $\{x^i, w^i\}_{i=1}^N$, where $N$ is the number of particles and $w^i$ denotes the weight of the $i$-th particle. The belief is formally denoted as: 
$
\bar{b}(x) \triangleq \frac{\sum_{i=1}^N w^i\delta(x^i-x)}{\sum_{i=1}^N w^i}.
$
Sparse-PFT uses sparse sampling \cite{Kearns02ml} to construct a search tree that, for each posterior belief node in the tree,  branches the entire action space and  samples $N^O$ observations. Each posterior belief node in the tree is updated using a particle filter. 

To facilitate the bounds estimation in our simplification framework with adaptive topology, we extend SparsePFT to AT-SparsePFT (Adaptive Topology SparsePFT), which works on the belief tree with augmented history as defined in~\eqref{eq:history-updater} and~\eqref{eq:history-update-fo}.
Given a topology $\tau$, AT-SparsePFT estimates 
the upper and lower bounds \eqref{eq:BoundsThm1_a}
over the optimal action-value function of the original (particle-belief) POMDP, 
 with formal finite-time guarantees. This is possible since these bounds correspond to the optimal action-value function of the corresponding topology, considering AOL and AFO settings. 
 We provide probabilistic guarantees on the estimation error of the estimated bounds by AT-SparsePFT for a given topology $\tau$ in the following theorem.

\begin{restatable}{theorem}{ATSparsePFTGuaranteeRestate}
	\label{theorem:ss-guarantee}
	Fix an arbitrary \mbox{$\lambda>0$}. Consider $N^O$ being the observation sampling number, $N$ being the state sampling number, and $C=\min\{N,N^O\}$.
	For every depth $d\in\{0,\dots,L-1\}$ and every action
	$a_d\in A$. The following event holds with probability at least  
	$
	1-2\,|A|\,(|A|C)^{\,L-d}\;
	\exp\!\Bigl(-\frac{C\lambda^{2}}{2V_{\max}^{2}}\Bigr)$:
	\begin{align}
		|\hat{ub}(\bar{b}_d,a_d,\tau)-{ub}(\bar{b}_d,a_d,\tau)| \leq t \quad , \quad |\hat{lb}(\bar{b}_d,a_d,\tau)-{lb}(\bar{b}_d,a_d,\tau)| \leq t,
	\end{align}
	where $t\triangleq\frac{(L-d)(L-d-1)}{2}\,\lambda$, and 
	$\hat{ub}$ and $\hat{lb}$ are the estimated bounds by SparsePFT for a given topology $\tau$ as
	 \begin{align}
		\hat{ub}(\bar{b}_0,a_0,\tau)\triangleq \hat{Q}^{\pi^{\text{AFO},\tau*}}(b_0,a_0) \quad,\quad \hat{lb}(\bar{b}_0,a_0,\tau)\triangleq \hat{Q}^{\pi^{\text{AOL},\tau*}}(b_0,a_0).
		\label{eq:EstBounds}
	\end{align}
%
%
\end{restatable}
\begin{proof}
	We provide the proof in Appendix~\ref{proof:theorem_2}. 
\end{proof}
We utilize these estimated bounds to identify the optimal action at the root belief node, same manner as shown in Fig.~\ref{fig:two-guarantees}. Based on the probabilistic guarantee, we can get the confidence level for  identifying the optimal action via the estimated bounds at the root belief node. In practice, we start from a highly-simplified initial topology, e.g., a completely open-loop setting, and utilize the AT-SparsePFT solver to estimate the bounds in \eqref{eq:ub-lb-simplification} with probabilistic guarantees (Theorem \ref{theorem:ss-guarantee}). In case of an overlap at the root of the tree, as shown in Fig.~\ref{fig:two-guarantees-left}, we refine the topology and repeat the process until we can determine the optimal action via  non-overlapping bounds, as shown in Fig.~\ref{fig:two-guarantees-right}. The topology transition process is described in Appendix~\ref{app:topology-transition}. Notably, the AT-SparsePFT solver caches and reuses the belief nodes remaining unchanged when transitioning to the next topology, which significantly reduces the computational overhead. 

However, the sparse-sampling-style solver is limited to a short planning horizon due to the exponential complexity on the action space. 
Next, we introduce an anytime solver that scales better to longer horizons.

\subsubsection{AT-POMCP.}
We adapt the MCTS-style solver POMCP~\cite{Silver10nips} to estimate the bounds used in the proposed simplification while adapting the topology in an anytime manner, namely AT-POMCP (Adaptive Topology POMCP). In AT-POMCP, the belief tree is constructed online according to the adaptive open-loop history under a given topology $\tau$. 
During each simulation, when expanding new nodes, the solver uses the adaptive history updater defined in \eqref{eq:history-updater} or \eqref{eq:history-update-fo} to associate the new node with its corresponding history. Inspired by the progressive widening methodology~\cite{Couetoux11iclio} in continuous MCTS~\cite{Sunberg18icaps}, we propose a similar mechanism to adapt online the topology in a progressive manner during the POMCP simulation. This means the topology will be updated once after a given  number of iterations following the procedure described in Appendix~\ref{app:topology-transition}. This budget of iterations is adaptively updated using progressing widening. A detailed algorithm is provided in Appendix~\ref{app:mcts-algorithm}, with lines $15-22$ showing the progressive adaptation of topology. As far as we know, this is the first work to adapt the topology of a belief tree in an anytime progressive manner.
We provide the following convergence theorem for AT-POMCP:
\begin{restatable}[Convergence of AT-POMCP]{theorem}{ATPomCPConvergenceRestate}
	\label{theorem:at-pomcp-convergence}
Considering a discrete POMDP, for a suitable choice of the UCB parameter $c$ and progressive widening parameter of AT-POMCP, the value function estimated by AT-POMCP, $\hat{V}^{\mathrm{AT}*}(b_0)$, converges in probability to the optimal value function, $V^*(b_0)$: $\hat{V}^{\mathrm{AT}*}(b_0) \xrightarrow{p} V^*(b_0)$.
\end{restatable}
\begin{proof}
	The proof is provided in Appendix~\ref{appendix:convegence-at-pomcp}.
\end{proof}

\noindent To summarize, in this section we proposed a 
methodology for adaptive open-loop simplification in POMDP planning with formal performance guarantees. Next, we leverage it to enable skipping replanning in POMDPs with formal guarantees.




\section{Skipping Replanning with Performance Guarantees}
Conventional online planners perform replanning at every time step after executing the optimal immediate action $a_0^*$, which was computed during the initial planning session at $t=0$. 
In this section, we introduce a novel framework that enables skipping replanning at certain steps during execution while maintaining formal performance guarantees. 
We formalize this approach as \textit{open-loop execution} of POMDPs with performance guarantees relative to standard online planners within the original POMDP framework, where agents replan after each execution step. This extends the \textit{adaptive open-loop planning} presented in Section~\ref{sec:open-loop-method} to \textit{open-loop execution}.
To our knowledge, this is a novel methodology to address the fundamental question of "when to replan"~\cite{Honda24icra}.

The proposed framework for skipping replanning in POMDPs comprises two main phases: planning and execution. The planning phase employs the simplification method introduced in Section~\ref{sec:open-loop-method}; the execution phase facilitates safe skipping of replanning. This framework is  outlined in Algorithm~\ref{alg1} of Appendix~\ref{appendix:algorithm-skip-replanning}.


The core idea underlying the decision to skip replanning at a \textit{future} time instant $k$, during the \textit{current} planning session at time instant $t=0$, involves reasoning about and bounding the corresponding action-value function $Q^{\pi*}(b_k,a_k)$ for different future posterior beliefs $b_k$ and actions $a_{k}$, with regard to the optimal policy sequence $\pi^*_{k+1:k+L}$. Conceptually, such bounds provide formal performance guarantees that, under suitable conditions, enable skipping replanning if the optimal immediate action $a_k^*$ can be deterministically determined at time instant $t=0$. As will be seen, performance guarantees must be evaluated sequentially for future time instants, as these guarantees are meaningful only when they hold for all preceding actions, allowing the agent to skip replanning in preceding sessions. In other words, the proposed framework can support skipping a replanning session at a future time instant $k$ as long as it can be skipped also in all preceding time instances until then. 



Our proposed framework for skipping replanning operates as an \textit{execution-time} decision that fully utilizes the time allocated to action execution and observation collection. Conventional online planners must pause while the agent executes actions and gathers observations. In contrast, our method operates in parallel with execution, thus avoiding additional overhead when assessing the feasibility of skipping replanning.
In scenarios where action execution demands significant time, this approach yields substantial simplifications not only at the planning level but also at the execution level.
Compared to alternative methods that reason if replanning can be skipped at each present time, given the corresponding posterior belief that is conditioned on the history at that time, our approach fully exploits the execution time for action performance and observation acquisition, resulting in reduced replanning overhead, i.e., upon receiving an observation from the environment, the decision of whether to skip replanning is immediate.

The core of our proposed framework is to check the performance guarantees for actions $a_k$ at future time steps with $k\in[1,L]$, as indicated in line 18 of Algorithm~\ref{alg1}.
For the $Q$-value of $a_k$ and posterior belief $b_k$ at the $k$-th future planning session, Theorem~\ref{theorem:open-loop2} provides the bounds: $\text{lb}(\tau, b_k,a_k) \leq Q^{\pi^*}(b_k,a_k) \leq \text{ub}(\tau, b_k,a_k)$, where  $Q^{\pi*}$ is the optimal $Q$-function at time $k$ with an $L$-steps planning horizon. However, the posterior belief $b_k$ is unknown at the current planning session (at time instant $t=0$). Given a topology $\tau$ with open-loop actions in the first $k$ steps and recursively assuming the performance guarantees hold for preceding actions $a_{1:k-1}$, we  now reformulate these bounds for \textit{any} future belief $b_k$ to depend solely on information available at $t=0$: 
\begin{align}
	\text{lb}^k(\tau, b_0,a_{0:k})&\!\leq\! \text{lb}(\tau, b_k,a_k) \!\leq\! Q^{\pi^*}(b_k,a_k) \! \leq \!
	\text{ub}(\tau, b_k,a_k) 
	\!\leq \!\text{ub}^{k}(\tau, b_0,a_{0:k}).
\label{eq:bound-replan}
\end{align}
We can utilize these bounds, at time $t=0$, to check overlap for different candidate future actions $a_k$ and determine the optimal action $a_k^*$, using the same principle as shown in Fig.~\ref{fig:two-guarantees} and discussed in Section~\ref{subsec:performance-guarantees}. We denote this process as checking \textit{Skipping Replanning Guarantees} (SRG).

Specifically, consider a set of topologies $\mathcal{T}^k$, where each topology $\tau$ incorporates open-loop simplifications in the first $k$ steps: $\mathcal{T}^k=\{\tau:\beta^{\tau}(h)=1, \forall h\in \mathcal{H}_{0:k-1}\}$.
We present the main theorem deriving the bounds in~\eqref{eq:bound-replan}: 
\begin{restatable}{theorem}{TheoremFourRestate}
	\label{theorem:bound-step-2}
	Consider the current time to be $0$ and a topology $\tau \in \mathcal{T}^k$. Assuming a positive $Q$-function, we have:   
	\begin{align}
		\mathrm{lb}^k(\tau, b_0,a_{0:k}) &\triangleq C_k(\CZ,\CX^R)\left( \tilde{Q}^{{\pi}^{\text{AOL},\tau*}}_{L+k}(b_0, a_{0:k-1},a_k) - \sum_{i=0}^{k-1}\mathbb{E}[r(b_i,a_i)] \right),
		\\
		\mathrm{ub}^{k}(\tau, b_0,a_{0:k}) &\triangleq \frac{1}{C_k(\CZ,\CX^R)}\! \left(\tilde{Q}^{{\pi}^{\text{AFO},\tau*}}_{L+k}(b_0, a_{0:k-1},a_k)\! - \! \sum_{i=0}^{k-1}\mathbb{E}[r(b_i,a_i)]\right)\!.
	\end{align}
%
Here, $\tilde{Q}^{{\pi}^{AOL,\tau*}}_{L+k}(b_0, a_{0:k-1},a_k)$ is obtained by extending the planning horizon to $L+k$,  enforcing the action sequence $a_{0:k-1}$ at the first $k$ steps, and following the policy $\pi$ afterwards, i.e.
	$\tilde{Q}^{\pi^{\tau*}}_{L+k}(b_0, a_{0:k-1},a_k) \triangleq \sum_{i=0}^k \emathbb{z_{1:i}} [r(b_i, a_i)] + \sum_{i=k+1}^{   k+L}\emathbb{z_{1:i}}\big[ r(b_i, \pi^{\tau*}_{i}(b_i))\big],$
	and 
$	C_k(\CZ,\CX^R) \triangleq  \prod_{j=1}^{k} \frac{\min_{z_j \in \CZ,x_j \in \CX^R} \{P(z_j|x_j)>0\}}{\max_{z_j \in \CZ,x_j \in \CX^R} \{P(z_j|x_j)>0\}},$ with $\CX^R$ being the reachable state space after executing $a_{0:k-1}$ from $b_0$, i.e.~$\CX^R = \CX_{\mathrm{reach}}(b_0, a_{0:k-1}) = \mathrm{supp} (\mathbb{P}(x_k|b_0,a_{0:k-1}))$,
and $ \sum_{i=0}^{k-1}\mathbb{E}[r(b_i,a_i)]$ corresponds to the expected sum of rewards when executing the action sequence $a_{0:k-1}$ from $b_0$.
\end{restatable}
\begin{proof}
	We provide the proof in Appendix~\ref{proof:theorem_3}. 
\end{proof}
 We now formulate conditions for skipping replanning with optimality guarantees.

\begin{proposition}
	\label{proposition:skip-replanning}
	Consider the current time to be $t=0$, the belief is $b_0$, and a topology $\tau \in \mathcal{T}^k$. For each $i\in[1,k]$, we can obtain $\mathrm{lb}^i(\tau, b_0,a_{0:i})$ and $\mathrm{ub}^i(\tau, b_0,a_{0:i})$ as in Theorem~\ref{theorem:bound-step-2}. 	
	If for each $i\in[1,k]$, we can identify the optimal action $a_i^*$, i.e., $a_i^*=\argmax_{a\in \CA} \mathrm{ub}^i(\tau, b_0,(a^*_{0:i-1},a))$ and 
	$\mathrm{lb}^i(\tau, b_0,(a^*_{0:i-1},a^*_i)) \geq \\ \mathrm{ub}^i(\tau, b_0,(a^*_{0:i-1},a'_i)),  \forall a'_i \neq a^*_i$, 
	as shown in Fig.~\ref{fig:two-guarantees-right}, then we can identify the optimal action sequence $a^*_{0:k}$ at time $t=0$ for all the possible future beliefs. Thus, we can directly execute $a^*_{0:k}$ and replanning can be safely skipped for all future planning sessions between time instances $[1,k]$.
\end{proposition}

However, the bounds presented in Theorem~\ref{theorem:bound-step-2} may become uninformative in certain scenarios, as  $C_k$ accounts for all possible observations.
To address this limitation, we now propose a refined variant that enables performance guarantee maintenance for skipping replanning across a broader range of scenarios.

We introduce a subset of observations $\bar{\mathcal{Z}}_{1:k}$, termed the \emph{allowed observation set} for skipping replanning, where $\forall i\in[1,k], \ \ \bar{\mathcal{Z}}_{i} \subseteq \mathcal{Z}$. We can tighten the bounds from Theorem \ref{theorem:bound-step-2} by constraining the factor $C_k$ with $\bar{\mathcal{Z}}_{1:k}$, which can potentially restore the performance guarantees in some scenarios where the original formulation provides too loose bounds.
For instance, in practice, the subset $\bar{\mathcal{Z}}_{1:k}$ can be constructed to include some of the most likely observations.

Given the allowed observation set $\bar{\mathcal{Z}}_{1:k}$, we can reformulate the bounds in \eqref{eq:bound-replan} to account for this subset as $\bar{lb}^{k}(\tau, b_0,a_{0:k},\bar{\mathcal{Z}}_{1:k})$ and $\bar{ub}^{k}(\tau, b_0,a_{0:k},\bar{\mathcal{Z}}_{1:k})$, which leads to the following theorem:
\begin{restatable}{theorem}{TheoremFiveRestate}
	\label{theorem:bound-step-2-z-region}
        Under the same conditions and definitions  as in Theorem \ref{theorem:bound-step-2}, but given a subset of observations $\bar{\mathcal{Z}}_{1:k} \subseteq \mathcal{Z}$, we have: 
			\begin{align*}
		 &\bar{lb}^{k}(\tau, b_0,a_{0:k},\bar{\mathcal{Z}}_{1:k}) 
		= {C}_k(\bar{\mathcal{Z}}_{1:k},\CX^R)\left( \tilde{Q}^{{\pi}^{\text{AOL},\tau*}}_{L+k}(b_0, a_{0:k-1},a_k) - \sum_{i=0}^{k-1}\mathbb{E}[r(b_i,a_i)] \right), 
		\\
		 &\bar{ub}^{k}(\tau, b_0,a_{0:k},\bar{\mathcal{Z}}_{1:k})\!
		= \frac{1}{{C}_k(\bar{\mathcal{Z}}_{1:k},\CX^R)}\! \left(\tilde{Q}^{{\pi}^{\text{AFO},\tau*}}_{L+k}(b_0, a_{0:k-1},a_k)\! -\! \sum_{i=0}^{k-1}\mathbb{E}[r(b_i,a_i)]\right),
	\end{align*}
	 where $C_k(\bar{\CZ}_{1:k},\CX^R)$ is defined in Theorem \ref{theorem:bound-step-2}.
\end{restatable}
\begin{proof}
	We provide the proof in Appendix~\ref{proof:theorem_4}.
\end{proof}

\begin{figure}[t]
	\centering
	\includegraphics[width=0.66\linewidth]{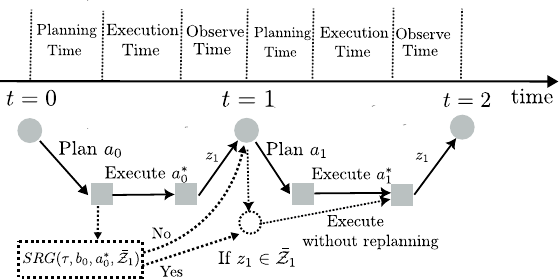}
	\caption{The process of skipping replanning based on the allowed observation set $\bar{\mathcal{Z}}_1$. The agent executes the action $a^*_0$ and checks if the observation $z_{1}$ belongs to the set $\bar{\mathcal{Z}}_{1}$. If it does, the agent skips replanning; otherwise, it triggers replanning.  The process to skip replanning with performance guarantee is shown in dotted lines, which can be conducted in parallel with the execution of action and observation, thus reducing the replanning overhead. 
	}
	\label{fig:skip-replanning-pipeline}
\end{figure}

Based on the reformulated bounds, we can now determine whether the performance guarantee holds given a subspace $\bar{\mathcal{Z}}_{1:k}$ at time $t=0$, i.e., if we can determine the optimal action $a_k$ based on the bounds from Theorem \ref{theorem:bound-step-2-z-region} with respect to $\bar{\mathcal{Z}}_{1:k}$. We denote such a check by the Boolean function $SRG(\tau, b_0, a^*_{0:k-1},a_k, \bar{\mathcal{Z}}_{1:k})$. 

Although the bounds in the above theorems may be computationally expensive or scenario-dependent, they can be computed in parallel with the action execution and observation collection, thus avoiding additional overhead, which forms the \textit{execution-time} planning. 
Fig.~\ref{fig:skip-replanning-pipeline} illustrates this process of \textit{execution-time} planning for the specific case when $k=1$. The agent checks the SRG with $\bar{\CZ}_1$ after the optimal action $a_0^*$ is identified. If SRG holds, i.e., $SRG(\tau, b_0, a^*_{0}, a_1, \bar{\mathcal{Z}_1})=\text{true}$, upon executing action $a_{0}$ and receiving observation $z_1$, the agent only needs to check if $z_1 \in \bar{\mathcal{Z}_1}$ to determine if replanning can be safely skipped. If $z_1 \in \bar{\mathcal{Z}}_1$, the agent can proceed executing $a_1$ without replanning; otherwise, replanning is triggered.


\section{Experiments}

We empirically evaluate our simplification methodology using practical solvers, AT-SparsePFT and AT-POMCP. Our experimental setup adopts the SparsePFT algorithm~\cite{Lim23jair} and POMCP~\cite{Silver10nips} as baselines. Results show that our approach preserves optimal decisions while substantially improving computational efficiency, yielding notable reductions in runtime without compromising quality.

We conduct the evaluation on the Beacon Navigation problem, where an agent is trying to reach a goal while avoiding obstacles under localization uncertainty. The robot receives noisy localization signals from beacons, reflecting real-world sensing conditions. This domain highlights the core challenge of planning under partial observability, requiring belief maintenance and informed action selection. It provides a practical yet tractable benchmark for evaluating belief space planning simplifications.

Detailed experimental settings are provided in Appendix~\ref{appendix:experimental-details}.

\subsection{Open-loop Simplification for Planning}
\paragraph{Bounds Estimation.}

We first present results from single-step simulation starting from a deterministic belief. Fig.~\ref{fig:qvalue-comparison} illustrates the estimated bounds obtained using our method compared with the $Q$-value estimated by standard sparse sampling. These results demonstrate that our proposed method effectively bounds the $Q$-value, thereby supporting our performance guarantees.

\begin{figure}[t]
    \centering
    \includegraphics[width=0.55\textwidth]{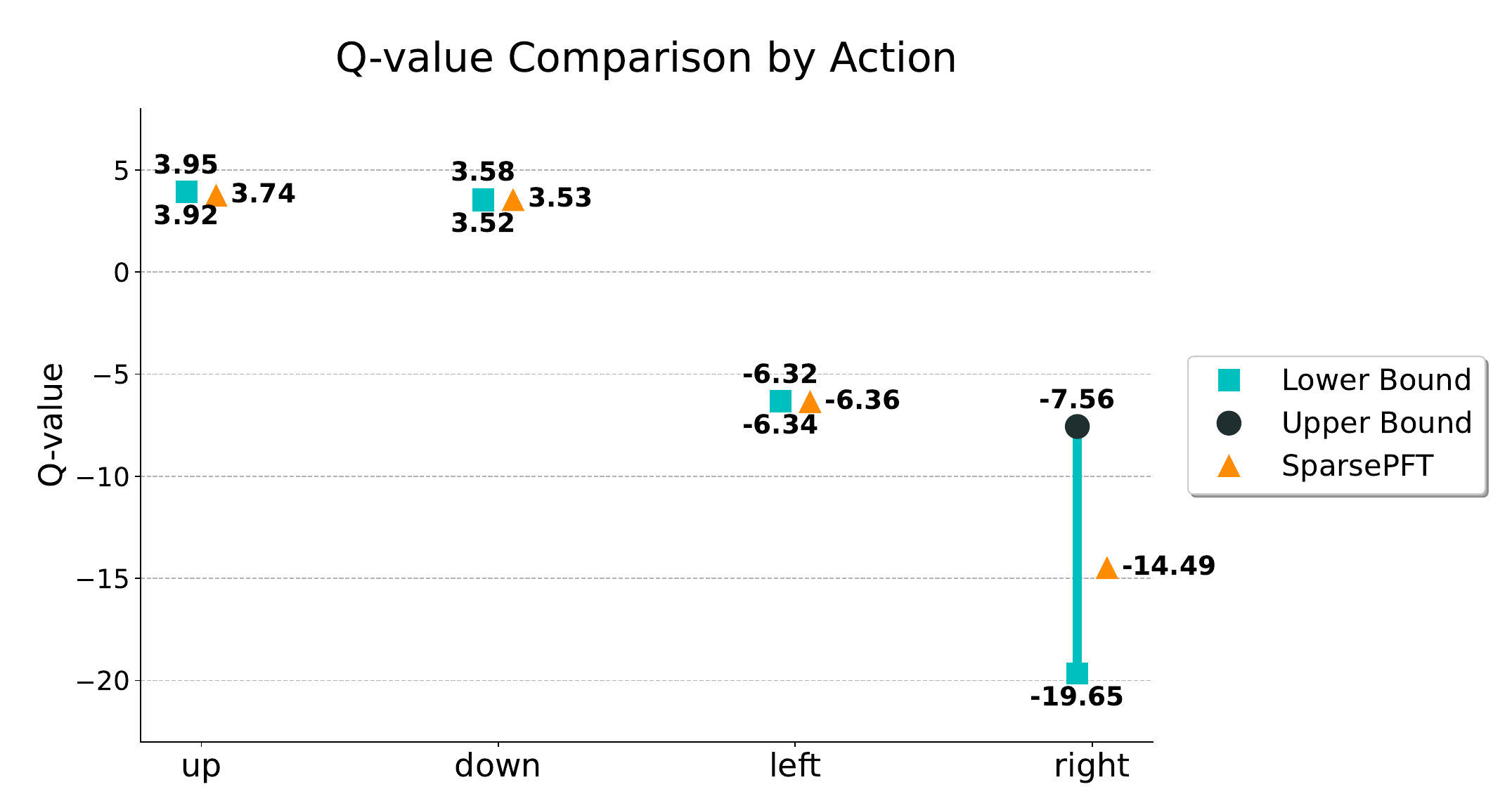}
    \caption{Distribution of estimated bounds. The upper and lower bounds are computed using our proposed method with open-loop simplification, AT-SparsePFT. The yellow triangle denotes the $Q$-value estimated by the standard SparsePFT. }
    \label{fig:qvalue-comparison}
\end{figure}

\paragraph{AT-SparsePFT Simulation Results.}
 We conducted simulations over 10 steps for the beacon navigation problem, with each simulation repeated 100 times to compute means and standard deviations. Table~\ref{tab:simulation-results} presents the results, comparing runtime and cumulative rewards after 10 steps. Our method, AT-SparsePFT, achieves a significant speedup of approximately $16\times$ while maintaining comparable solution quality. The mean cumulative rewards obtained by our method closely match those of the standard SparsePFT approach, while providing substantial computational efficiency improvements. These results validate our claims regarding POMDP simplification with performance guarantees.

\begin{table}[t]
    \centering
    \resizebox{0.6\textwidth}{!}{
    \begin{tabular}{lcccc}
        \hline
        Method & Returns  & Runtime (s) & Speedup Ratio \\
        \hline
        SparsePFT  & $12.64 $ & $345.9 $ & $1.0\times$ \\
        AT-SparsePFT (ours)           & $12.50 $ & $\mathbf{20.66 }$ & $\mathbf{16.7\times}$ \\
        \hline
    \end{tabular}
    }
    \vspace{+0.1cm}
    \caption{Cumulative rewards and runtime of the baseline SparsePFT method and our proposed method, AT-SparsePFT. The speedup ratio is computed as the ratio of baseline to proposed runtime. Results are averaged over 100 independent 10-step simulations.}
    \label{tab:simulation-results}
\end{table}

\paragraph{AT-POMCP Simulation Results.} We further conduct a comprehensive empirical evaluation of the proposed method estimated by the MCTS-style estimator (as introduced in Section~\ref{subsec:solvers}) on the beacon navigation task over a planning horizon of 10 steps. Each experimental configuration is repeated 20 times to compute mean cumulative rewards and standard deviations. The computational time budget is progressively increased from $50\,\text{ms}$ to $1\,\text{s}$ for both the baseline POMCP method and our proposed approach.  
As to the progressive topology adaptation, the topology is set to be adapted every $100$ simulations.

Table~\ref{tab:mcts-compare-with-time-budget} presents a comparison of cumulative rewards achieved by the baseline POMCP method and our proposed anytime solver AT-POMCP under the same computational budgets. Our approach consistently achieves superior cumulative rewards across all tested time budgets, with improvements ranging from $6.84\%$ to $10.24\%$. Given that the beacon navigation domain exhibits a sparse reward structure, these improvements in average cumulative rewards are noteworthy and statistically meaningful.

\begin{table}[t]
\centering
\resizebox{\textwidth}{!}{%
\begin{tabular}{lcccccc}
\hline
Time Budget & 50ms & 80ms & 100ms & 300ms & 500ms & 1.0s \\
\hline
POMCP & $6.988 \pm 0.352$ & $7.082 \pm 0.270$ & $7.030 \pm 0.165$ & $7.14 \pm 0.229$ & $7.15 \pm 0.284$ & $7.347 \pm 0.277$ \\
AT-POMCP(ours) & $\textbf{7.533} \pm 0.481$ & $\textbf{7.566} \pm 0.533$ & $\textbf{7.698} \pm 0.494$ & $\textbf{7.825} \pm 0.655$ & $\textbf{7.883} \pm 0.353$ & $\textbf{8.018} \pm 1.526$ \\
Improvement & 7.80\% & 6.84\% & 9.52\% & 9.57\% & 10.24\% & 9.12\% \\
\hline
\end{tabular}
}
\caption{Cumulative reward comparison across computational budgets. Results are reported as mean $\pm$ standard deviation over 20 trials.}
\label{tab:mcts-compare-with-time-budget}
\end{table}

To assess the computational advantage of our approach, we compare the runtime required for the baseline POMCP algorithm and the proposed method when both achieve similar mean cumulative reward. 
Table~\ref{tab:mcts-same-reward} reports the experimental results: the baseline attains an average return of $7.456$ after $2.5$s of computation, whereas our method reaches a slightly higher return of $7.533$ in only $0.05$s. This corresponds to an approximately $50\times$ speedup. The dramatic reduction in runtime while preserving reward quality validates our claim of the open-loop simplification of POMDP planning which yields substantial efficiency gains while preserving quality.


\begin{table}[t]
    \centering
    \resizebox{0.7\textwidth}{!}{
    \begin{tabular}{lccccc}
        \hline
        Method & Returns  & Runtime (s) & Speedup Ratio \\
        \hline
        POMCP & $7.456 \pm 0.276$  & $2.5$ & $1.0\times$ \\
        AT-POMCP(ours)   & $7.533 \pm 0.481$ & $\mathbf{0.05}$ & $\mathbf{50\times}$ \\
        \hline
    \end{tabular}
    }
    \caption{Runtime comparison between the baseline POMCP and the proposed AT-POMCP when operating at the same performance level (average cumulative reward). The speedup ratio is computed as the baseline runtime divided by the proposed method runtime; returns are averaged over 20 independent trials.}
    \label{tab:mcts-same-reward}
\end{table}

Table~\ref{tab:mcts-same-reward} demonstrates that our algorithm not only yields higher cumulative rewards but, more importantly, converges at a markedly faster rate than the baseline POMCP. The reward attained by the proposed method after merely $50\,\text{ms}$ of computation would require roughly $2500\,\text{ms}$ for the baseline to match—an approximately $50\times$ acceleration in convergence.  
This pronounced speedup provides strong empirical evidence that the convergence rate of the simplified POMDP planner is dramatically superior, while simultaneously delivering equal solution quality.  Consequently, the results support our theoretical analysis that our proposed open-loop simplification leads to improved planning efficiency in the online anytime solvers.


\subsection{Experiments on Skipping Replanning}
This section evaluates the effectiveness of our proposed framework in skipping replanning while preserving performance guarantees. We demonstrate this through a specific scenario, tunnel navigation, and show that skipping replanning is both feasible and safe. In this setup, similar to the beacon navigation problem, the agent navigates through a tunnel to reach the goal with noisy observations from beacons. We consider a positive reward function here. This scenario exemplifies a typical setting for open-loop planning, as previously explored in~\cite{Hauser11icml}.

\begin{figure}[t]
    \centering
    \begin{subfigure}[b]{0.4\textwidth}
        \centering
        \includegraphics[width=\textwidth]{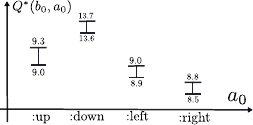}
        \caption{
            }
        \label{fig:replan-bounds-a0}
    \end{subfigure}
    \qquad
    \begin{subfigure}[b]{0.4\textwidth}
        \centering
        \includegraphics[width=\textwidth]{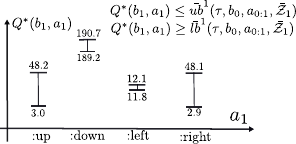}
        \caption{
            }
        \label{fig:replan-bounds-a1}
    \end{subfigure}
    \caption{Skipping replanning with guarantees: 
    (a) Planning with the proposed adaptive open-loop simplification at $t=0$. It shows $ub(\tau, b_0,a_0)$ and $lb(\tau, b_0,a_0)$ for different $a_0$. $a_0^*$ is identified. (b) Bounding $Q^{*}(b_1,a_1)$ at $t=0$ for any posterior belief $b_1$ that corresponds to future observations in $\bar{\mathcal{Z}}_1$. The figure shows $\bar{ub}^{1}$ and $\bar{lb}^{1}$ defined in Theorem~\ref{theorem:bound-step-2-z-region}. Here $a_1^*$ can be provably determined at $t=0$ for all realizations of the future observation $z_1$ in $\bar{\mathcal{Z}}_1$. }
    \label{fig:replan-bounds-two}
\end{figure}

Fig.~\ref{fig:replan-bounds-two} illustrates the process of verifying the proposed framework of skipping replanning with formal guarantee.  Using our proposed framework to safely skip replanning, at $t=0$, the optimal action $a_0^*$ can be identified and the future action $a_1^*$ can be provably determined for all possible future observations $z_1 \in \bar{\mathcal{Z}}_1$, where $\bar{\mathcal{Z}}_1$ includes the four most likely observations.

Table~\ref{tab:replan-results} summarizes the results of our skipping replanning method. We compare AT-SparsePFT with and without our skipping replanning method. Our method successfully skips replanning in $24\%$ of the steps while maintaining performance guarantees. The cumulative reward obtained by our method is almost identical to that of the method without skipping replanning due to the stochastic properties of the simulations, thereby validating the performance guarantees of our approach. The skipping ratio serves as a key metric, highlighting the efficiency gains of our framework at the overall execution level, in addition to the planning-level simplifications introduced in Section~\ref{sec:open-loop-method}.


\begin{table}[t]
    \centering
    \resizebox{0.8\textwidth}{!}{
    \begin{tabular}{lcccc}
        \hline
        Method & Returns  & Skipping Replanning Ratio \\
        \hline
        AT-SparsePFT w/o Skipping Replanning& $ 235.7$  & $0\%$ \\
        AT-SparsePFT w/ Skipping Replanning & $231.3$ & $\mathbf{24\%}$ \\     
        \hline
    \end{tabular}
    }
    \caption{Comparison of cumulative rewards and the ratio of steps where replanning can be safely skipped with performance guarantees. Results are averaged over 100 independent 5-step simulations.}
    \label{tab:replan-results}
\end{table}

\section{Conclusions}

In this paper, we introduced a novel framework for adaptive open-loop simplification of POMDPs, enabling a significant reduction in computational complexity while maintaining formal performance guarantees. By leveraging a topology-based belief tree, our approach adaptively interleaves open-loop and closed-loop planning, providing efficiently computable bounds that guarantee identification of the optimal action in the original POMDP. We propose practical solvers for our adaptive simplification, AT-SparsePFT and AT-POMCP. Furthermore, we proposed a principled method for safely skipping replanning, supported by theoretical guarantees on multi-step open-loop action sequences. Empirical results demonstrate substantial speedup with provable guarantees, highlighting the practicality of our approaches. 

\newpage


\bibliographystyle{splncs04}
\bibliography{../references/refs}
\appendix
\newenvironment{proofsketch}{%
  \par\noindent\textit{Proof sketch.}\ }{%
  \hfill\(\square\)\par}

\onecolumn

\clearpage
\pagenumbering{roman}
\setcounter{page}{1}

\renewcommand{\thesection}{\Roman{section}}
\renewcommand{\thesubsection}{\thesection.\Alph{subsection}}
\setcounter{section}{0} 

\setcounter{equation}{0}

\begin{center}
    \LARGE \textbf{Open-loop POMDP Simplification and Safe Skipping of Replanning\\with Formal Performance Guarantees\\} 
    \textbf{Supplementary Material}
\end{center}
\section{Proof of Theorem~\ref{theorem:open-loop2}}
\label{proof:theorem_1}




\begin{proof}
We prove the upper and lower bounds separately by comparing the policy spaces and information available to each policy.

\paragraph{Upper Bound: $Q^{\policyfstar}(b_0,a_0) \leq Q^{\policy^{\text{AFO},\tau*}}(b_0,a_0)$}
The Adaptive Fully-Observable (AFO) policy, $\pi^{\text{AFO},\tau}$, operates in an information space that is at least as rich as that of the standard POMDP policy, $\pi^*$. At any belief node where the topology indicator $\beta^{\text{AFO},\tau}=1$, the AFO policy has access to the true state $x_t$, which is strictly more informative than the belief $b_t$ available to the standard policy. A policy with access to more information can always achieve a value at least as high as a policy with less information, as it can simply choose to ignore the extra information and replicate the less-informed policy's strategy.

Consequently, the set of all AFO policies, $\Pi^{\text{AFO},\tau}$, can achieve any value that the set of standard POMDP policies, $\Pi^{\text{Full}}$, can. Maximizing over these policy spaces, the optimal value for the AFO policy must be greater than or equal to the optimal value for the standard POMDP:
\[
Q^{\policyfstar}(b_0,a_0) = \max_{\pi \in \Pi^{\text{Full}}} Q^{\pi}(b_0,a_0) \leq \max_{\pi' \in \Pi^{\text{AFO},\tau}} Q^{\pi'}(b_0,a_0) = Q^{\policy^{\text{AFO},\tau*}}(b_0,a_0).
\]

\paragraph{Lower Bound: $Q^{\pi^{\text{AOL},\tau*}}(b_0,a_0) \leq Q^{\policyfstar}(b_0,a_0)$}
The Adaptive Open-Loop (AOL) policy, $\pi^{\text{AOL},\tau}$, is more constrained than the standard POMDP policy. At any belief node where $\beta^{\text{AOL},\tau}=1$, the AOL policy must select an action without conditioning on the most recent observation $z_t$. This means the policy space for AOL, $\Pi^{\text{AOL},\tau}$, is a subset of the full POMDP policy space, $\Pi^{\text{Full}}$. Since the optimization for the AOL policy is performed over a smaller set of policies, its optimal value cannot exceed that of the standard POMDP:
\[
Q^{\pi^{\text{AOL},\tau*}}(b_0,a_0) = \max_{\pi \in \Pi^{\text{AOL},\tau}} Q^{\pi}(b_0,a_0) \leq \max_{\pi^* \in \Pi^{\text{Full}}} Q^{\pi^*}(b_0,a_0) = Q^{\policyfstar}(b_0,a_0).
\]
This relationship can be seen from the Bellman equations. The standard POMDP performs a maximization after observing $z_t$, while the AOL policy must maximize before the expectation over $z_t$ when in an open-loop step. By Jensen's inequality, we have:
	\begin{align}
        \nonumber
		&\emathbb{x_{t}|h_t^-}\emathbb{z_{t}|x_{t}}\max_{a_{t}} \big[r(x_{t},a_{t})+\emathbb{x_{t+1}|h_t,a_t}\emathbb{z_{t+1}|x_{t+1}}V(b_{t+1})\big] \\ & \geq\max_{a_{t}}\emathbb{x_{t}|h_t^-}\emathbb{z_{t}|x_{t}} \big[r(x_{t},a_{t})+\emathbb{x_{t+1}|h_t,a_t}\emathbb{z_{t+1}|x_{t+1}}V(b_{t+1})\big].
	\end{align} 
The left side corresponds to the standard closed-loop update, and the right side corresponds to the open-loop update, confirming the sub-optimality.

This establishes the claimed upper and lower bounds.
\end{proof}

\section{Proof of Theorem~\ref{theorem:ss-guarantee}}
\label{proof:theorem_2}
\begin{proof}
    We now provide the proof for the lower bound on the error of the sparse-sampling estimator for the adaptive open-loop policy, $\Delta\hat{lb}$. The proof proceeds by induction on the depth $d$ of the belief tree, from $d=L$ down to $d=0$. The value function $V_{\max}$ is defined as $V_{\max} = L \cdot R_{\max}$.

    \textbf{Base Case ($d=L$):}
    At the maximum depth $L$, the Q-value is simply the immediate expected reward, as there are no future steps.
    $Q^{\pi^{\text{AOL},\tau*}}(b_L, a_L) = \mathbb{E}_{x_L|b_L}[r(x_L, a_L)]$.
    The sparse-sampling estimator $\hat{lb}(b_L, a_L)$ for this value is an average over samples from the belief $b_L$. By Hoeffding's inequality, for any $\lambda' > 0$, the probability of the estimation error exceeding $\lambda'$ is bounded:
    \[
    \Pr\left[ |\hat{lb}(b_L, a_L) - lb(b_L, a_L)| > \lambda' \right] \le 2 \exp\left(-\frac{2N(\lambda')^2}{R_{\max}^2}\right),
    \]
    where $N$ is the number of state samples from $b_L$. 

    \textbf{Inductive Step:}
    Assume that for any depth $d' > d$, any belief $b_{d'}$, and any action $a_{d'}$, the following holds with high probability:
    $
    |\hat{lb}(b_{d'}, a_{d'}, \tau) - lb(b_{d'}, a_{d'}, \tau)| \le \frac{(L-d')(L-d'-1)}{2}\lambda.
    $

    Now, consider depth $d$. The lower bound $Q$-value, $lb(b_d, a_d, \tau) = Q^{\pi^{\text{AOL},\tau*}}(b_d, a_d)$, is given by the Bellman equation for the adaptive open-loop policy. We analyze the two cases based on the topology indicator $\beta^{\text{AOL},\tau}$.

    \textit{Case 1: Closed-loop step ($\beta^{\text{AOL},\tau}(h_d) = 0$)}
    \begin{align}
        lb(b_d, a_d, \tau) = \mathbb{E}_{x_d|b_d}[r(x_d, a_d)] +  \mathbb{E}_{z_{d+1}|b_d, a_d} \left[ \max_{a_{d+1}} lb(b_{d+1}, a_{d+1}, \tau) \right].
    \end{align}
    The estimator $\hat{lb}(b_d, a_d, \tau)$ is computed by sampling $N^O$ observations $z_{d+1}^{(j)}$ and recursively calling the estimator on the resulting beliefs $b_{d+1}^{(j)}$:
    \begin{align}
        \hat{lb}(b_d, a_d, \tau) = \mathbb{E}_{x_d|b_d}[r(x_d, a_d)] + \frac{1}{N^O} \sum_{j=1}^{N^O} \max_{a_{d+1}} \hat{lb}(b_{d+1}^{(j)}, a_{d+1}, \tau).
    \end{align}
    Let $Y(z_{d+1}) = \max_{a_{d+1}} lb(b_{d+1}, a_{d+1}, \tau)$ and $\hat{Y}(z_{d+1}) = \max_{a_{d+1}} \hat{lb}(b_{d+1}, a_{d+1}, \tau)$. The error can be bounded by the sum of a sampling error and the propagated error from the next level:
    \begin{align}
        \Delta\hat{lb}(b_d, a_d, \tau) \le&  \left| \mathbb{E}[\hat{Y}] - \frac{1}{N^O}\sum \hat{Y}_j \right| +  \left| \mathbb{E}[Y] - \mathbb{E}[\hat{Y}] \right| \\
        \le&  \left| \mathbb{E}[\hat{Y}] - \frac{1}{N^O}\sum \hat{Y}_j \right| +  \mathbb{E}[|Y - \hat{Y}|].
    \end{align}
    The first term is a sampling error, bounded by $\lambda$ with high probability using Hoeffding's inequality. The second term is bounded by the inductive hypothesis: $ \mathbb{E}[|Y - \hat{Y}|] \le  \frac{(L-d-1)(L-d-2)}{2}\lambda$.
    Thus, the total error is bounded by $\lambda + \frac{(L-d-1)(L-d-2)}{2}\lambda = \frac{(L-d)(L-d-1)}{2}\lambda$.

    \textit{Case 2: Open-loop step ($\beta^{\text{AOL},\tau}(h_d) = 1$)}
    \begin{align}
        lb(b_d, a_d, \tau) = \mathbb{E}_{x_d|b_d}[r(x_d, a_d)] +  \max_{a_{d+1}} \mathbb{E}_{x_{d+1}|b_d, a_d} \left[ lb(b_{d+1}, a_{d+1}, \tau) \right].
    \end{align}
    The estimator $\hat{lb}(b_d, a_d, \tau)$ samples $N$ next states $x_{d+1}^{(j)}$:
    \begin{align}
        \hat{lb}(b_d, a_d, \tau) = \mathbb{E}_{x_d|b_d}[r(x_d, a_d)] + \max_{a_{d+1}}\frac{1}{N} \sum_{j=1}^{N}  \hat{lb}(b_{d+1}^{(j)}, a_{d+1}, \tau).
    \end{align}
    The logic is identical to Case 1, but the expectation is over next states $x_{d+1}$ instead of observations $z_{d+1}$. The error is bounded similarly.

    Since $C = \min\{N, N^O\}$, the error bound is largely dominated by the worst case.
    To complete the proof, we use a union bound over all possible histories and actions. 
    Applying the union bound over all actions at depth $d$ and all nodes in the subsequent subtree, the probability at depth $d$ for a given action $a_d$ is bounded by $2(|A|C)^{L-d} \exp(-\frac{C\lambda^2}{2V_{\max}^2})$. Taking another union bound for all $|A|$ actions at depth $d$ gives the final probability in the theorem statement.

    For the upper bound, we have a similar result as in the lower bound case.
\end{proof}

\section{Bound Analysis of $ub$ and $lb$}
\label{appendix:bound-analysis}
\paragraph{Complexity Analysis}
The computational complexity of solving $Q^{{\pi}^{\text{AOL},\tau*}}$ and $Q^{\pi^{\text{AFO},\tau*}}$ is primarily determined by the topology $\tau$, where a larger number of belief nodes with simplification yields more substantial complexity reduction.  For each belief node with simplification (i.e., $\beta^{\tau}=1$), the complexity of a single-step Bellman update with state-dependent reward is reduced from $\mathcal{O}(|X||Z||A|)$ to $\mathcal{O}(|X||A|)$, consistent with the cancellation of expectation over observations.

\paragraph{Convergence.} When the topology is switched, certain belief nodes are reverted to a closed-loop setting by appropriately introducing observations (as discussed in Section~\ref{subsec:definition}). In the extreme case, all belief nodes will be switched to a closed-loop setting, which corresponds to a standard POMDP belief tree. Consequently, both the upper and lower bounds converge to the optimal $Q$-function of the full POMDP.

\paragraph{Monotonicity.}
We establish monotonicity with respect to transitions toward topologies containing a greater number of closed-loop belief nodes. Let $\tau$ denote the original topology and $\tau'$ the modified topology. For the upper bound, we have $Q^{{\pi}^{AFO,\tau,*}}(b_0,a_0) \geq Q^{{\pi}^{AFO,\tau',*}}(b_0,a_0)$. For the lower bound, $Q^{{\pi}^{AOL,\tau*}}(b_0,a_0) \leq Q^{{\pi}^{AOL,\tau'*}}(b_0,a_0)$. These inequalities demonstrate that the bounds monotonically tighten as the topology transitions to include more closed-loop nodes.

The proof of the monotonicity is based on the same method to prove Theorem~\ref{theorem:open-loop2}. The key idea is that the AFO policy with more nodes using full observability can achieve at least as high a value as the AFO policy with fewer nodes with full observability, and similarly for the AOL policy.

\paragraph{Parallel Calculation of the Upper and Lower Bounds.}
The policy $\policy^{AFO,\tau*}$ can be computed in parallel with the policy $\policy^{AOL,\tau*}$, as the upper and lower bounds are independent of each other. This parallel computation can significantly reduce the overall computation time required to obtain both bounds.

\section{Cross-topology Transition}
\label{app:topology-transition}
When bounds defined in \eqref{eq:ub-lb-simplification} under topology $\tau$ overlap, we have to explore an alternative topology $\tau'$ to achieve non-overlapping bounds through an iterative process, which continues until we identify the optimal action.
We propose an incremental refinement method that ensures monotonic bound tightening and asymptotic convergence to the full POMDP solution.
Since $\tau$ actually means a pair of $\tau_U$ and $\tau_L$, we will transition both topologies simultaneously during the topology adaptation process from $\tau$ to $\tau'$, i.e., $\tau' = (\tau_U',\tau_L')$.

The transition process first selects some belief nodes in belief tree $\mathbb{T}^{\tau}$ to transition from open-loop step to closed-loop step, creating an updated topology $\tau'$ with a new belief tree $\mathbb{T}^{\tau}$. 
This transition typically introduces observations at the selected belief nodes and expands the belief tree structure with new nodes.

Then, the indicator function $\beta^{\tau'}$ is updated to: (1) enforce closed-loop planning for the chosen nodes, (2) preserve the original planning mode for the unchanged nodes, and (3) inherit the mode from the nodes in the previous topology for newly generated belief nodes. This approach ensures consistent topology evolution.

\paragraph{Belief nodes caching.} To reduce the computational overhead, we cache belief computations during planning. During topology transitions, unchanged nodes are retrieved directly from the cache, while only selected nodes require recomputation.

\section{Algorithm of AT-POMCP}
\label{app:mcts-algorithm}

The proposed anytime solver, AT-POMCP, is  given by Algorithm~\ref{alg:aol-pomcp}.

\begin{algorithm}[thb]
\caption{Topology-based MCTS-style Anytime Solver}
\label{alg:aol-pomcp}
\begin{algorithmic}[1]
\STATE \textbf{Input:} state $s$, history $h$, planning horizon $L$, simulation index $i$; topology $\tau$, topology adaptation index $j$, topology progressive adaptation parameter $\alpha$ and $k$.
\STATE \textbf{Output:} return estimate $R$ and update belief tree.

\STATE \textbf{Simulate}$(s,h,\mathrm{depth},i)$
\IF{$\mathrm{depth} > L$}
    \STATE \textbf{return} $0$
\ENDIF
\IF{$h \notin T$}
    \FOR{$a \in \mathcal{A}$}
        \STATE $T(ha) \gets (N_{\mathrm{init}}(ha), V_{\mathrm{init}}(ha), \emptyset)$
    \ENDFOR
    \STATE \textbf{return} $\mathrm{Rollout}(s,h,\mathrm{depth})$
\ENDIF
\STATE $a \gets \arg\max_b \left[ V(hb) + c \sqrt{\frac{\log N(h)}{N(hb)}} \right]$
\STATE $(s',o,r) \sim \mathcal{G}(s,a)$
\STATE // --- Progressive topology adaptation ---
\IF{$i\leq k \cdot j^{\alpha}$}  
    \STATE $\tau \gets \tau$
\ELSE
    \STATE $j \gets j+1$
    \STATE // --- Randomly set some belief nodes' indicator function $\beta^{\tau}$ to $0$ ---
    \STATE $\tau \gets \mathrm{RandomTopoTransition}(\tau)$
\ENDIF
\STATE Update History $h'$ by~\eqref{eq:history-updater} or~\eqref{eq:history-update-fo} based on topology $\tau$
\STATE $R \gets r + \mathrm{Simulate}(s',h',\mathrm{depth}+1, i)$
\STATE $B(h) \gets B(h) \cup \{s\}$
\STATE $N(h) \gets N(h) + 1$
\STATE $N(ha) \gets N(ha) + 1$
\STATE $V(ha) \gets V(ha) + \dfrac{R - V(ha)}{N(ha)}$
\STATE \textbf{return} $R$
\end{algorithmic}
\end{algorithm}

\section{Convergence of AT-POMCP: Proof of Theorem~\ref{theorem:at-pomcp-convergence}}
\label{appendix:convegence-at-pomcp}
We provide the convergence guarantee of AT-POMCP in the following theorem.

\begin{proofsketch}
	The proof starts from a fixed topology version of AT-POMCP. If the topology is fixed, the convergence of AT-POMCP directly follows from the convergence of POMCP~\cite{Silver10nips}: $\hat{V}^{\tau*}(b_0) \xrightarrow{p} V^{\tau*}(b_0)$.
	Then, the proof will focus on the progressive adaptation of topology, which will affect the UCB action selection process and its convergence. 

	Consider the belief node with history $\tilde{h}_t$ \footnote{For simplicity, we sometimes refer to belief node with history $\tilde{h}_t$ as belief node $\tilde{h}_t$.}, where it expands action branches, and leads to child nodes, among which we consider a specific history $\tilde{h}_{t+1}^-$.  We assume after some iterations, the progressive adaptation of topology will change the topology from $\tau$ to $\tau'$ and modify the indicator function at this belief node from $\beta^\tau(\tilde{h}_{t+1}^-)=1$ to $\beta^{\tau'}(\tilde{h}_{t+1}^-)=0$. This means the belief node $\tilde{h}_{t+1}^-$ will expand observation branches after the topology adaptation.  
	This adaptation will affect the UCB action selection process at the belief node $\tilde{h}_t$ and the value estimation at $\tilde{h}_{t+1}^-$. We will analyze the convergence of value estimation at these two belief nodes after the topology adaptation. If the convergence at these two nodes hold, the convergence of other nodes will follow.

	\paragraph{1. We examine the average return at the node $\tilde{h}_{t+1}^-$.}
    Denote $N(\tilde{h}_{t+1}^-)$ to be the number of simulations at belief node with history $\tilde{h}_{t+1}^-$. 
    Then, the average return at $\tilde{h}_{t+1}^-$ under the new topology will become: 
    \begin{align}
    \bar{G}^{\tau', N(\tilde{h}_{t+1}^-)}(\tilde{h}_{t+1}^-) = \frac{G^{\tau, N_{\beta=1}(\tilde{h}_{t+1}^-)}(\tilde{h}_{t+1}^-) + G^{\tau', N_{\beta=0}(\tilde{h}_{t+1}^-)}(\tilde{h}_{t+1}^-)}{N(\tilde{h}_{t+1}^-)},
    \end{align} 
    where $N_{\beta=1}(\tilde{h}_{t+1}^-)$ and $N_{\beta=0}(\tilde{h}_{t+1}^-)$ are the number of simulations at $\tilde{h}_{t+1}^-$ before and after the topology adaptation, and $N(\tilde{h}_{t+1}^-)=N_{\beta=1}(\tilde{h}_{t+1}^-)+N_{\beta=0}(\tilde{h}_{t+1}^-)$. The return will include some simulations under the old topology. Since the simulation number $N_{\beta=1}(\tilde{h}_{t+1}^-)$ will not increase anymore and is finite, as we increase the number of simulations, the ratio of $N_{\beta=1}(\tilde{h}_{t+1}^-)/N(\tilde{h}_{t+1}^-)$ will converge to $0$. Thus, the average return will converge: 
    \begin{equation}
        \label{eq:return-converge-new-topo}
    \bar{G}^{\tau', N(\tilde{h}_{t+1}^-)}(\tilde{h}_{t+1}^-) \to \bar{G}^{\tau',N_{\beta=0}(\tilde{h}_{t+1}^-)}(\tilde{h}_{t+1}^-).
    \end{equation}
	This indicates that the estimation at $\tilde{h}_{t+1}^-$ will converge to the new topology part. 
	\paragraph{2. We examine the UCB action selection at belief node $\tilde{h}_t$.} At belief node $\tilde{h}_t$, we are interested in the estimated optimal Q-function under the new topology $\tau'$. This part is based on the convergence of UCB derived by~\cite{Kocsis06ecml} and~\cite{Auer02ml}. 
    
    The analysis aims to show that the expected number of sub-optimal action selections,  $\mathbb{E}(N_{\mathrm{Sub-opt}})$, is bounded. Based on~\cite{Auer02ml}'s proof of 
    Theorem 1, we can split the choosing of the sub-optimal action into three events: \textbf{(A)} The optimal branch is not explored enough and has a much lower average return than the theoretical one, \textbf{(B)} The average return of the suboptimal branch is much higher than the theoretical one, \textbf{(C)} The theoretical value of the  optimal branch is not large enough to distinguish the optimal action. (See Equation (7)-(9) in~\cite{Auer02ml}). For event (C), our case follows the same theoretical level analysis as~\cite{Auer02ml}.

    For our case, we will show that the topology adaptation will not affect event (A) and (B) in the estimation level, by considering the following three cases:
	\begin{itemize}
		\item[i.] If the given history $\tilde{h}_{t+1}^-$ is not in the optimal action branch under the both old and new topologies, based on the proof in the previous part, the average return after the topology adaptation will converge to the new topology part as shown in Equation~\eqref{eq:return-converge-new-topo}. Thus it will eventually converge to the theoretical value under the new topology, not breaking case (B).
        
		\item[ii.] If the given history $\tilde{h}_{t+1}^-$ is in the optimal action branch under the old and new topologies, based on the proof of converge in the first part, the average return after the topology adaptation will converge to the new topology part, as shown in Equation~\eqref{eq:return-converge-new-topo}. This will not break case (A).
        
		\item[iii.] If the given history $\tilde{h}_{t+1}^-$ is in the optimal action branch under the new topology but not under the old topology, this will affect the UCB action selection at the time of topology adaptation. But if we keep increasing the number of simulations, due to the exploration bonus term, the UCB action selection will eventually explore the branch with $\tilde{h}_{t+1}^-$ under the new topology enough times. As long as the branch with $\tilde{h}_{t+1}^-$ under the new topology is explored enough times, based on the proof of converge in the first part, the average return after the topology adaptation will converge to the new topology part (see Equation~\eqref{eq:return-converge-new-topo}), which will become the optimal branch. So, this case will not break case (A).
        
        \item[iv.] An additional case could be as follows: If the given history $\tilde{h}_{t+1}^-$ is in the optimal action branch under the old topology but not under the new topology. However, this will not happen because the topology transition will monotonically tighten the bounds toward the value function of the original POMDP, as shown in Appendix~\ref{appendix:bound-analysis}.
	\end{itemize}

    So far, we have shown that all the cases will not affect the boundedness of $\mathbb{E}(N_{\mathrm{Sub-opt}})$, then the rest of the proof can directly follow~\cite{Kocsis06ecml}.
	This means that the topology adaptation will not affect the convergence of AT-POMCP if adapting from topology $\tau$ to topology $\tau'$, as: $\hat{V}^{\tau'*}(b_0) \xrightarrow{p} V^{\tau'*}(b_0)$.   	
	
	Using the progressive topology adaptation, eventually the topology of the belief tree will converge to the original POMDP topology $\tau_0$, without any simplification. Repeating the process in the above proof, we can show that $\hat{V}^{\tau_0*}(b_0) \xrightarrow{p} V^{\tau_0*}(b_0)=V^*(b_0)$. Thus, the value function estimated by AT-POMCP, $\hat{V}^{\mathrm{AT}*}(b_0)$ converges in probability to the optimal value function $V^*(b_0)$. This finishes the proof.

\end{proofsketch}

\section{Algorithm of Safe Skipping Replanning}
\label{appendix:algorithm-skip-replanning}
The algorithm of safe skipping replanning is given by Algorithm~\ref{alg1}.

\begin{algorithm}[ht]
\caption{Safe Skipping Replanning in POMDPs}
\label{alg1}
\resizebox{0.98\linewidth}{!}{%
\begin{minipage}{\linewidth}
\begin{algorithmic}[1]
\STATE \textbf{Input:} Initial belief $b_0$, initial topology $\tau_0$
\WHILE{not in a terminal state}
    \STATE // --- Planning Phase to find optimal action $a_0^*$ ---
    \STATE $\tau \gets \tau_0$
    \REPEAT
        \STATE For each action $a \in \mathcal{A}$, compute bounds $ub(\tau, b, a)$ and $lb(\tau, b, a)$.
        \IF{an optimal action $a_0^*$ is identified}
            \STATE Get corresponding policy ${\pi}^{AOL,\tau*}$
            \STATE \textbf{break}
        \ELSE
            \STATE Refine topology $\tau \gets \tau'$ to tighten bounds.
        \ENDIF
    \UNTIL{optimal action $a_0^*$ is found}
    
    \STATE // --- Execution Phase ---
    \FOR{$k = 1, 2, \dots, L$}
        \IF{the $k$-th step in ${\pi}^{AOL,\tau*}$ is open-loop}
            \STATE Let $a_k$ be the $k$-th action in ${\pi}^{AOL,\tau*}$
            \IF{$SRG(\tau, b_0, a^*_{0:k-1}, a_k, \bar{\mathcal{Z}_{1:k}})=\text{true}$}
                \STATE \textbf{continue}. //(Skip replanning)
            \ENDIF
        \ENDIF
        \STATE \textbf{break}. //(Trigger replanning)
    \ENDFOR
\ENDWHILE
\end{algorithmic}
\end{minipage}%
}
\end{algorithm}

\section{Proof of Theorem~\ref{theorem:bound-step-2}}
\label{proof:theorem_3}

\begin{proof}

We will prove the lower bound first. The proof for the upper bound follows a similar approach. The core of the proof is to relate the posterior belief $b_k$ at a future time step $k$ to the belief propagated open-loop from the initial belief $b_0$.

Let $b_k(x_k) = P(x_k | h_k)$ be the posterior belief at time $k$, where the history is $h_k = \{a_{0:k-1}, z_{1:k}\}$. Let $b_k^-(x_k) = P(x_k | h_k^-)$ be the belief propagated up to time $k$ before incorporating the observation $z_k$, where $h_k^- = \{a_{0:k-1}, z_{1:k-1}\}$. The relationship is given by Bayes' rule:
\begin{align}
    b_k(x_k) = \frac{P(z_k|x_k) b_k^-(x_k)}{P(z_k|h_k^-)}.
\end{align}
By recursively applying this, we can relate $b_k(x_k)$ to the initial belief $b_0(x_0)$ and the open-loop propagated belief $\phi^p(b_0, a_{0:k-1})(x_k)$:
\begin{align}
    b_k(x_k) = \frac{\prod_{j=1}^k P(z_j|x_j)}{\prod_{j=1}^k P(z_j|h_j^-)} \phi^p(b_0, a_{0:k-1})(x_k).
\end{align}
Let $c_j = \frac{\min_{z_j\in\mathcal{Z}, x_j\in\mathcal{X}} P(z_j|x_j)}{\max_{z_j\in\mathcal{Z}, x_j\in\mathcal{X}} P(z_j|x_j)}$, for $P(z_j|x_j)>0$. We can bound the ratio of the belief distributions:
\begin{align}
    \prod_{j=1}^k c_j \cdot \phi^p(b_0, a_{0:k-1})(x_k) \le b_k(x_k) \le \frac{1}{\prod_{j=1}^k c_j} \cdot \phi^p(b_0, a_{0:k-1})(x_k).
\end{align}
This gives us \begin{align}
C_k(\mathcal{Z}, \mathcal{X}) \phi^p(b_0, a_{0:k-1}) \le b_k \le \frac{1}{C_k(\mathcal{Z}, \mathcal{X})} \phi^p(b_0, a_{0:k-1}).
\label{eq11}
\end{align}

Now, let's analyze the Q-value. Assume state-dependent rewards and positive Q-value. From Theorem~\ref{theorem:open-loop2}, the original Q value is lower-bounded by the AOL Q-value, and we can have the following relationship by directly applying ~\eqref{eq11} :
\begin{align}
    Q^{\pi^*}(b_k, a_k) \ge Q^{\pi^{AOL,\tau,*}}(b_k, a_k) \ge C_k(\mathcal{Z}, \mathcal{X}) Q^{\pi^{\text{AOL},\tau*}}(\phi^p(b_0, a_{0:k-1}), a_k).
\end{align}
The term $Q^{\pi^{\text{AOL},\tau*}}(\phi^p(b_0, a_{0:k-1}), a_k)$ is the expected value of a plan of length $L$ starting from belief $\phi^p(b_0, a_{0:k-1})$. This can be rewritten in terms of a longer plan starting from $b_0$:
\begin{align}
    Q^{\pi^{\text{AOL},\tau*}}(\phi^p(b_0, a_{0:k-1}), a_k) = \tilde{Q}^{\pi^{\text{AOL},\tau*}}_{L+k}(b_0, a_{0:k-1},a_k) - \sum_{i=0}^{k-1} \mathbb{E}[r(b_i, a_i)],
\end{align}
where $\tilde{Q}^{\pi^{\text{AOL},\tau*}}_{L+k}(b_0, a_{0:k-1},a_k)$ is the value of a plan of horizon $L+k$ from $b_0$ where the first $k$ actions are fixed to $a_{0:k-1}$.
Combining these gives the lower bound:
\begin{align}
    \text{lb}^k(\tau, b_0, a_{0:k}) = C_k(\mathcal{Z}, \mathcal{X}) \left( \tilde{Q}^{\pi^{\text{AOL},\tau*}}_{L+k}(b_0, a_{0:k-1},a_k) - \sum_{i=0}^{k-1} \mathbb{E}[r(b_i, a_i)] \right).
\end{align}
The proof for the upper bound, $\text{ub}^k(\tau, b_0, a_{0:k})$, follows symmetrically using the upper bound on the belief ratio and the upper bound from Theorem~\ref{theorem:open-loop2}. 

Finally, the state space $\CX$ can be tighten to be the reachable space $\CX^R$.
This completes the proof.
\end{proof}

\section{Proof of Theorem~\ref{theorem:bound-step-2-z-region}}
\label{proof:theorem_4}
\begin{proof}
The proof follows the same structure as the proof of Theorem~\ref{theorem:bound-step-2}. The key difference lies in the bounding of the posterior belief $b_k$.

In the proof of Theorem~\ref{theorem:bound-step-2}, the belief ratio is bounded over the entire observation space $\mathcal{Z}$. For this theorem, we are given that the future observations $z_k$ will fall within the allowed observation sets $\bar{\mathcal{Z}}_k$.
Therefore, the belief ratio bound can be tightened by replacing the full observation space $\mathcal{Z}$ with the subset $\bar{\mathcal{Z}}_k$ when defining the constant factor. The ratio of belief distributions is now bounded by:
\begin{align}
    {C}_k(\bar{\mathcal{Z}}_k, \mathcal{X}^R) \phi^p(b_0, a_{0:k-1}) \le b_k \le \frac{1}{{C}_k(\bar{\mathcal{Z}}_k, \mathcal{X}^R)} \phi^p(b_0, a_{0:k-1}),
\end{align}
where ${C}_k(\bar{\mathcal{Z}}_k, \mathcal{X}^R)$ is defined as in Theorem~\ref{theorem:bound-step-2-z-region} but with the `min` and `max` operations taken over the observation subset $\bar{\mathcal{Z}}_k$.

The remainder of the proof follows directly from the steps in the proof of Theorem~\ref{theorem:bound-step-2}, by substituting $C_k(\mathcal{Z}, \mathcal{X}^R)$ with the tighter factor ${C}_k(\bar{\mathcal{Z}}_k, \mathcal{X}^R)$. This yields the desired bounds $\bar{lb}^{k}$ and $\bar{ub}^{k}$.
\end{proof}

\section{Experimental Details}
\label{appendix:experimental-details}
\paragraph{Beacon Navigation Problem.} 
\begin{figure}[ht]
    \centering
    \includegraphics[width=0.45\textwidth]{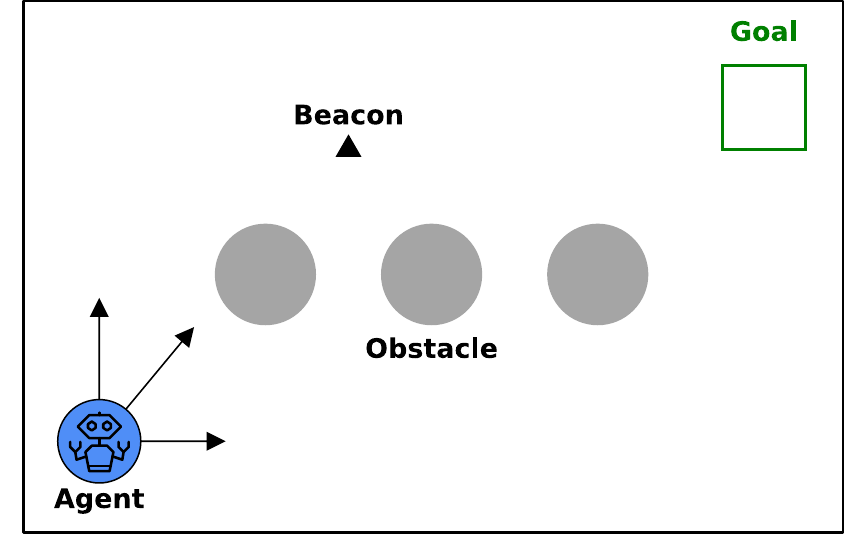}
    \caption{Beacon navigation problem.}
    \label{fig:beacon-exp}
\end{figure}
The beacon navigation problem is a POMDP where an agent navigates a grid world to reach a target beacon while avoiding obstacles. The agent can move in four directions (up, down, left, right) and can observe the distance to the beacon. The goal is to reach the beacon while maximizing the cumulative reward. The observation is based on the beacons, where the observation noise is defined as the distance to the beacon plus some noise. Figure~\ref{fig:beacon-exp} illustrates the problem setup.

The transition model is defined as follows:
\begin{align}
    P(x'|x,a) = \begin{cases}
        p^T_{int} & \text{if } x' = x + a \text{ and } x' \text{ is within bounds,} \\
        p^T_{adj} & \text{if } \  |x'-x-a|=1, \\
        p^T_{stay} & \text{if } x'=x, \\
        0 & \text{otherwise.}
    \end{cases}
\end{align}
Here, $p^T_{int}$ is the probability of moving to the intended cell, $p^T_{adj}$ is the probability of moving to an adjacent cell, and $p^T_{stay}$ is the probability of staying in the same cell. And $p^T_{int} + p^T_{adj} + p^T_{stay} = 1$.

The observation model is defined as:
\begin{align}
    P(z|x) = \begin{cases}
        1-p^O_{error} & \text{if } z=x \text{ and within beacon range,} \\
        \frac{p^O_{error}}{4} & \text{if } |z-x|=1 \text{ and within beacon range,}\\
        0 & \text{otherwise.}
    \end{cases}
\end{align}
Here, the factor $p^O_{error}$ represents the observation noise, which is based on the distance to the nearest beacon: $p^O_{error}=\min\{0.9, 1- d_{beacon}*0.15\}$, where $d_{beacon}$ is the distance to the nearest beacon. 

The reward function is defined as:
\begin{align}
    r(x,a) = 1_{x\in X^{goal}} \cdot r^{goal} + 1_{x\in X^{obstacle}} \cdot r^{obstacle} + r^{step} + r^{dis},
\end{align}
where $1_{x\in X^{goal}}$ is an indicator function that returns 1 if the agent is in the goal state set ($X^{goal}$), $r^{goal}$ is the reward for reaching the goal, $1_{x\in X^{obstacle}}$ is an indicator function that returns 1 if the agent is in an obstacle state set ($X^{obstacle}$), $r^{obstacle}$ is the penalty for hitting an obstacle, $r^{step}$ is a small negative reward for each step taken, and $r^{dis}$ is a reward based on the distance to the goal ($d_{goal}$) as: $r^{dis}=\frac{15}{1+d_{goal}}$.

In the experiments, we set the parameters as follows: $p^T_{int}=0.5$, $p^T_{adj}=0.2$, $p^T_{stay}=0.3$, $p^O_{error}=0.1$, $r^{goal}=200$, $r^{obstacle}=-30$, and $r^{step}=-0.5$. The grid size is $20\times 20$ with one beacon at $(3,3)$ and obstacles at $(2,3), (2,4), (9,3)$. The goal is at $(7,5)$. The agent starts at $(1,3)$.

The experiments are conducted on Ubuntu 22.04 with an Intel i9-9820X CPU and 64GB RAM. The code is implemented in Julia.

\section{Open-loop Simplification Experiment}
This section presents detailed results from the open-loop simplification experiment, including cumulative reward analysis and discussion of result standard deviation.
\subsection{AT-SparsePFT Experiment Figure}
This section provides additional visualization of the cumulative rewards at each step for both the baseline SparsePFT method and our proposed AT-SparsePFT method, as shown in Figure~\ref{fig:cumlative-reward-appendix}. 
\begin{figure}[h]
    \centering
    \includegraphics[width=0.45\textwidth]{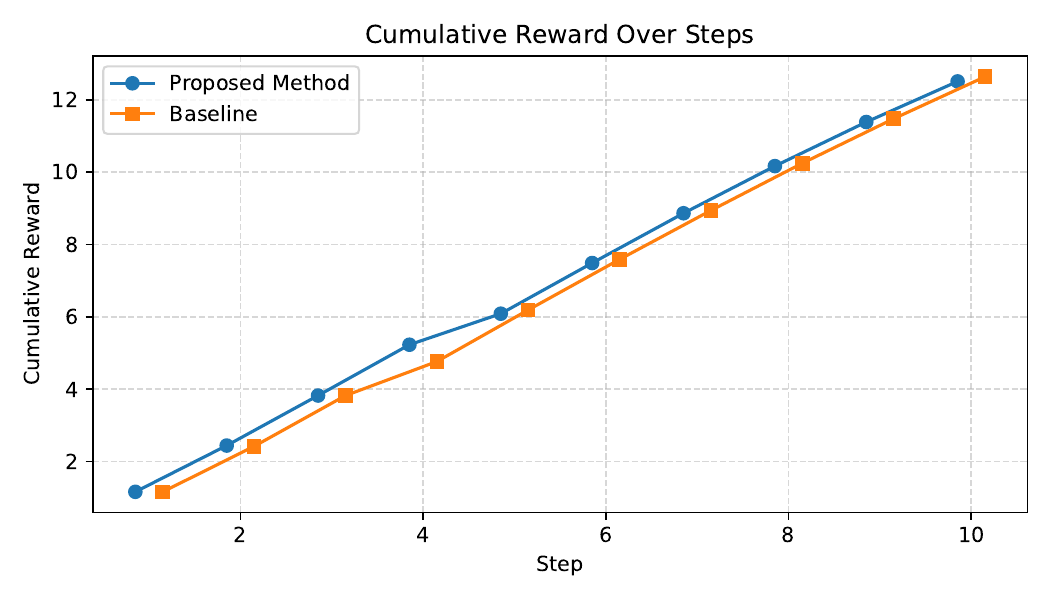}
    \vspace{-0.1cm}
    \caption{Cumulative rewards at each step for the baseline, SparsePFT, and our proposed method, AT-SparsePFT. Results are averaged over $100$ simulations. 
    }
    \label{fig:cumlative-reward-appendix}
\end{figure}

Fig.~\ref{fig:cumlative-reward-appendix} illustrates the cumulative rewards at each time-step for both the baseline sparsePFT and our proposed AT-SparsePFT. Results represent averages across $100$ independent simulations.
 Our proposed method achieves performance nearly identical to the baseline, 
 but shows a significant speedup while adapting topology online, empirically validating our theoretical performance guarantees.

\subsection{Results with Standard Deviation Analysis}
Table~\ref{tab:simulation-results-appendix} presents the cumulative rewards and runtime performance for the baseline sparse sampling (SS) method and our proposed approach (including mean and std of the results). Our method demonstrates a speedup ratio of $16.7\times$ compared to the baseline SS method while maintaining comparable cumulative rewards.

Both methods exhibit similar levels of standard deviation in cumulative rewards, with our proposed method demonstrating marginally lower variance. Given the stochastic nature of the problem, which incorporates noise in both transition and observation models, elevated standard deviation values are expected. Notably, our method maintains performance consistency comparable to the baseline, validating the proposed performance guarantees. The observed standard deviation in both methods is primarily attributed to a single collision event among 100 simulation runs, arising from the inherent stochasticity in the motion and observation models.

\begin{table}[htb]
    \centering
    \resizebox{0.65\textwidth}{!}{
    \begin{tabular}{lcccc}
        \hline
        Method & Returns  & Runtime (s) & Speedup Ratio \\
        \hline
        Baseline (SparsePFT)  & $12.64\pm6.02 $ & $345.9\pm6.4 $ & $1.0\times$ \\
        AT-SparsePFT            & $12.50\pm 5.42 $ & $\mathbf{20.66 \pm 5.0}$ & $\mathbf{16.7\times}$ \\
        \hline
    \end{tabular}
    }
    \caption{Cumulative rewards and runtime of the baseline SparsePFT method and our proposed method. The speedup ratio represents the ratio of baseline runtime to proposed method runtime. Results report mean values and standard deviations (std) computed over 100 independent 10-step simulation trials.}
    \label{tab:simulation-results-appendix}
\end{table}

\subsection{AT-POMCP Experiments}

Figure~\ref{fig:mcts-rewards-growing-time} provides a visual summary of the experiment results for the baseline POMCP and our proposed AT-POMCP method.

\begin{figure}[htb]
    \centering
    \includegraphics[width=0.65\textwidth]{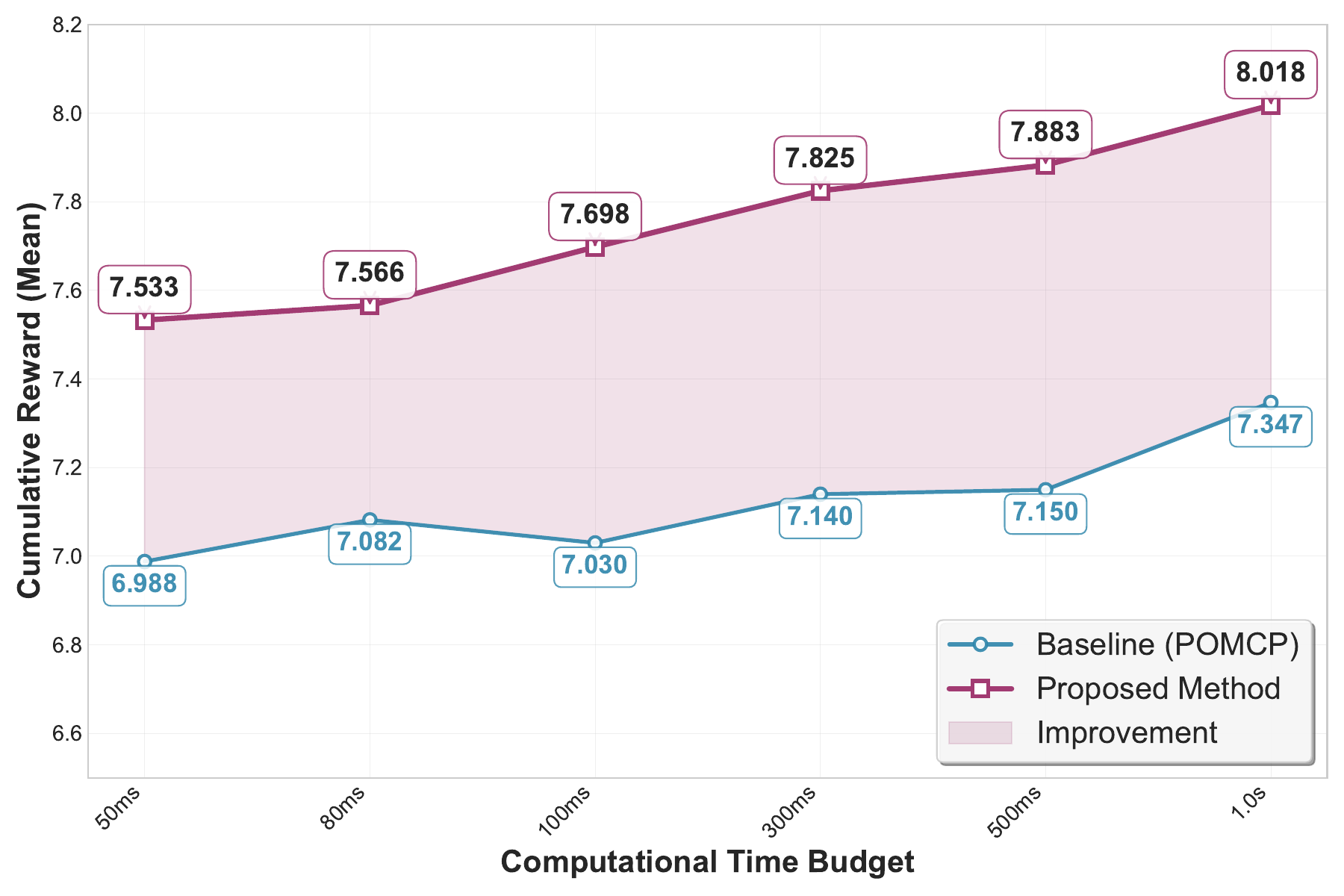}
    \caption{cumulative reward comparison between the baseline POMCP and the proposed method. The plot shows mean cumulative rewards for budgets ranging from $50\,\text{ms}$ to $1\,\text{s}$, highlighting consistent performance gains for the proposed method.}
    \label{fig:mcts-rewards-growing-time}
\end{figure}

\section{Replanning Skipping Experiments}
This section presents a runtime analysis incorporating realistic execution and observation time.



\subsection{Runtime analysis for skipping replanning check}

\begin{table}[htb]
\centering
    \resizebox{0.99\textwidth}{!}{
\begin{tabular}{lcccc}
\hline
Method & Planning Time (s) & Replanning Time (s) & Execution Time (s) & Total Time (s) \\
\hline
Skip Replanning (Ours) &  220.64 & 80.70 & 150 & \textbf{230.70} \\
Baseline (Always Replan) &  90.58 & 90.58 & 150 & 240.58 \\
\hline
\end{tabular}
    }
\caption{Runtime comparison between our replanning skip method and baseline approach over 100 trials. Execution and observation time per time-step: $1.5$s.}
\label{tab:replanning-runtime-appendix}
\end{table}

We evaluate a scenario where action execution and observation acquisition require $1.5$s per time-step. Table~\ref{tab:replanning-runtime-appendix} presents the runtime analysis for our replanning skip approach. Our method achieves a total runtime of $230.70$s, comprising $220.64$s for planning (including $80.70$s for replanning phases), compared to the baseline's $240.58$s total runtime.

Our approach achieves a $24\%$ replanning skip rate, indicating that replanning can be safely omitted in approximately one-quarter of all decision steps. While this theoretically corresponds to a $24\%$ reduction in replanning overhead, the current implementation realizes a smaller improvement due to unoptimized code. Future optimization efforts are expected to enhance runtime performance.





\end{document}